\documentclass[twoside]{article}

\usepackage[accepted]{aistats2025}
\usepackage{array}
\usepackage{dsfont}
\usepackage{amsthm}
\usepackage{amsmath}
\usepackage{amssymb}
\usepackage{mathtools}
\usepackage{bm}
\usepackage{pifont}
\usepackage{bbding}
\usepackage{hyperref}
\usepackage[ruled, vlined]{algorithm2e}
\usepackage{algpseudocode}
\usepackage{thmtools}
\usepackage{hyperref}
\usepackage[dvipsnames]{xcolor}
\usepackage[unq]{unique}
\definecolor{Bleu}{RGB}{30,144,245}
\definecolor{Red}{HTML}{FF617B}
\hypersetup{colorlinks,citecolor=Bleu,linkcolor=red}
\usepackage{todonotes}

\theoremstyle{plain}

\newtheorem{remark}{Remark}

\newtheorem{theorem}{Theorem}
\newtheorem{lemma}{Lemma}
\newtheorem{definition}{Definition}

\newcommand{\todoc}[1]{\todo[color=red!40, inline,caption={}]{CK: \small #1}}
\newcommand{\todoe}[1]{\todo[color=green!40, inline]{EK: \small #1}}

\renewcommand{\omega}{w}
\let\leq\leqslant
\let\geq\geqslant
\newcommand{\C}{\subset}
\newcommand{\T}{\intercal}
\DeclareMathOperator{\m}{m}
\DeclareMathOperator{\M}{M}
\newcommand{\eps}{\varepsilon}
\newcommand{\bP}{\mathbb{P}}

\newcommand{\bE}{\mathbb{E}}

\newcommand{\bR}{\mathbb{R}}

\newcommand{\bN}{\mathbb{N}}

\newcommand{\cN}{\mathcal{N}}

\newcommand{\cS}{{S}} 
\newcommand{\cA}{\mathcal{A}}
\newcommand{\cB}{\mathcal{B}}

\newcommand{\cD}{\mathcal{D}}
\newcommand{\cO}{\mathcal{O}}
\newcommand{\cP}{\mathcal{P}}
\newcommand{\cX}{\mathcal{X}}
\newcommand{\cI}{\mathcal{I}}

\newcommand{\cE}{\mathcal{E}}

\newcommand{\veps}{\varepsilon}

\newcommand{\iid}{{ \it i.i.d }}

\DeclareMathOperator*{\argmax}{argmax}
\DeclareMathOperator*{\argmin}{argmin}

\newcommand{\equi}{\Longleftrightarrow}
\DeclareMathOperator{\mh}{{{m}}}
\DeclareMathOperator{\Mh}{{{M}}}

\newcommand{\muh}{\ensuremath{{\widehat{\mu}}}} 

\newcommand{\vmu}{\ensuremath{{\mu}}} 

\newcommand{\vmuh}{\ensuremath{{\widehat{\mu}}}} 

\newcommand{\lp}{\ensuremath{\left(}}
\newcommand{\rp}{\ensuremath{\right)}}

\renewcommand{\complement}{\mathsf{c}}

\newcommand{\dimvec}{\ensuremath{d}}
\newcommand{\dimfeat}{\ensuremath{h}}

\newcommand{\gege}{{{GEGE}}}
\newcommand{\dimspan}[1]{\text{dim(span(}\{ x_i: i \in #1\}))}

\newcommand{\gegeFb}{{{GEGE}}}
\newcommand{\gegeFc}{{{GEGE}}}

\newcommand{\optDesign}{\hyperref[alg:est_gaps, citecolor=.]{$\mathrm{OptEstimator}$}}

\newcommand{\myeqref}[1]{Eq.~\eqref{#1}}

\newcommand{\myfig}[1]{Fig.\ref{#1}}
\newcommand{\myalg}[1]{Algorithm~\ref{#1}}

\newcommand{\algauername}{APE}
\newcommand{\pal}{PAL} 
\newcommand{\rank}[1]{\ensuremath{\text{rank}(#1)}}

\usepackage{float}
\usepackage{enumerate}
\usepackage{subcaption}

\usepackage{pgfplots}
\usepackage{tikz}
\usetikzlibrary{calc,intersections,arrows.meta, plotmarks}
\tikzstyle{every picture}+=[remember picture]
\tikzstyle{na} = [baseline=-.5ex]
\usepackage[font=small,labelfont=bf]{caption}
\definecolor{yellow2}{rgb}{0.8862745 , 0.84313726, 0.}
\definecolor{red2}{rgb}{0.9607843 , 0.3137255 , 0.07450981}
\definecolor{green3}{rgb}{0.        , 0.43529412, 0.52156866}
\definecolor{blue}{rgb}{0.08627451211214066, 0.125490203499794, 0.23529411852359772}
\usetikzlibrary{shapes.misc}
\usepackage{booktabs}
\usepackage{stmaryrd}
\usepackage[dvipsnames]{xcolor}
\usepackage[round]{natbib}

\begin{document}

%

%
\renewcommand{\todoe}[1]{}
\renewcommand{\todoc}[1]{}
\twocolumn[
\aistatstitle{Bandit Pareto Set Identification in a Multi-Output Linear Model}
\aistatsauthor{Cyrille Kone$^1$ \And Emilie Kaufmann$^1$ \And  Laura Richert$^2$}
\aistatsaddress{$^1$ Univ. Lille, Inria, CNRS, Centrale Lille, UMR 9198-CRIStAL, F-59000 Lille, France \\ 
	$^2$ Univ. Bordeaux, Inserm, Inria, BPH, U1219, Sistm, F-33000 Bordeaux, France} ]
\runningauthor{Cyrille Kone, \;Emilie Kaufmann, \;Laura Richert} 
\begin{abstract}
We study the Pareto Set Identification (PSI) problem in a structured multi-output linear bandit model. In this setting, each arm is associated a feature vector belonging to $\bR^\dimfeat$, and its mean vector in $\bR^d$ linearly depends on this feature vector through a common unknown matrix $\Theta \in \bR^{\dimfeat \times \dimvec}$. The goal is to identify the set of non-dominated arms by adaptively collecting samples from the arms. We introduce and analyze the first optimal design-based algorithms for PSI, providing nearly optimal guarantees in both the fixed-budget and the fixed-confidence settings. Notably, we show that the difficulty of these tasks mainly depends on the sub-optimality gaps of $\dimfeat$ arms only. 
Our theoretical results are supported by an extensive benchmark on synthetic and real-world datasets. 
\end{abstract}

\section{INTRODUCTION}
\label{sec:intro}

A multi-armed bandit is a stochastic game where an agent faces $K$ distributions (or arms) whose means are unknown to her. When the distributions are scalar-valued, the agent faces two main tasks: regret minimization and pure exploration. In the former, the agent aims at maximizing the sum of observations collected along its trajectory \citep{lattimore_bandit_2020}. In pure exploration, the agent has to solve a stochastic optimization problem after some steps of exploration, and it does not suffer any loss during exploration \citep{bubeck_pure_2008}. Examples of pure exploration tasks include best arm identification in which the goal is to find the arm with the largest mean \citep{audibert_best_2010}, thresholding bandit \citep{locatelli_optimal_2016}, or combinatorial bandits \citep{combBandit}, to name a few. 

In this paper, we are interested in the less common setting where the rewards are $\bR^d$-valued, with $d>1$. Different pure exploration tasks have been considered in this context, e.g., finding the set of feasible arms, i.e., arms whose mean satisfies some constraints~\citep{katz-samuels_feasible_2018}, or a feasible arm maximizing a linear combination of the different criteria~\citep{katz-samuels_top_2019,faizal2022constrained}. Finding appropriate constraints is not always possible in practical problems, and our focus is on the identification of the Pareto set, that is, the set of arms whose means are not uniformly dominated by that of any other arm, a setting first studied by~\citep{zuluaga_active_201,auer_pareto_2016}. We note that a regret minimization counterpart of this problem has been considered by~\citep{drugan_designing_2013}. 


Pareto set identification can be relevant in many real-world problems where there are multiple, possibly conflicting objectives to optimize simultaneously. Examples include monitoring the energy consumption and runtime of different algorithms (see our use case in Section~\ref{sec:expe}) or identifying a set of interesting vaccines by observing different immunogenicity criteria (antibodies, cellular response, that are not always correlated, as exemplified by \cite{kone2023adaptive}). In both cases, there could be many arms with a few descriptors of the different arms (e.g., vaccine technology, doses, injection times). By incorporating such arm features in the model, we expect to reduce substantially the number of samples needed to identify the Pareto set.

In this work, we incorporate some structure in the PSI identification problem through a multi-output linear model, formally described in Section~\ref{sec:setting}. In this model, each of the $K$ arms whose means are in $\bR^d$ is described by a feature vector in $\bR^h$, $h>1$. We propose the \gege{} algorithm, which combines a G-optimal design exploration mechanism with an accept/reject mechanism based on the estimation of some notion of sub-optimality gap. \gege{} can be instantiated in both the fixed-budget setting (given at most $T$ samples, output a guess of the Pareto set minimizing the error probability) and the fixed-confidence setting (minimize the number of samples used to guarantee an error probability smaller than some prescribed $\delta$). Through a unified analysis, we show that in both cases, the sample complexity of \gege{}, that is, the number of samples needed to guarantee a certain probability of error, scales only with the $h$ smallest sub-optimality gaps. This yields a reduction in sample complexity due to the structural assumption. 
Finally, we empirically evaluate our algorithms with extensive synthetic and real-world datasets and compare their performance with other state-of-the-art algorithms.

\noindent{\bf Related work} When $d=1$ and the feature vectors are the canonical basis of $\bR^K$, PSI coincides with the best arm identification problem, that has been extensively studied in the literature both in the fixed-budget \citep{audibert_best_2010,karnin_almost_2013,carpentier_tight_2016} and the fixed-confidence settings \cite{kalyanakrishnan_pac_2012,jamieson_lil_2013}. For sub-Gaussian distributions, the sample complexity is known to be essentially characterized (up to a $\log(K)$ factor in the fixed-budget setting) by a sum over the $K$ arms of the inverse squared value of their \emph{sub-optimality gap}, which is their distance to the (unique) optimal arm. In the fixed-confidence setting and for Gaussian distributions, there are even algorithms matching the minimal sample complexity when $\delta$ goes to zero, which takes a more complex, non-explicit form (e.g., \cite{garivier_optimal_2016,you23a}).    

Still, when $d=1$ but for general features in $\bR^h$, our model coincides with the well-studied linear bandit model (with finitely many arms), in which the best arm identification task has also received some attention.  It was first studied by \cite{soare} in the fixed-confidence setting, who established the link with optimal designs of experiments \citep{pukelsheim_opt_design}, showing that the minimal sample complexity can be expressed as an optimal (XY) design. The authors proposed the first elimination algorithms where, in each round the surviving arms are pulled according to some optimal designs and obtained a sample complexity scaling in $(h/\Delta_{\min}^2)\log(1/\delta)$ where $\Delta_{\min}$ is the smallest gap in the model. 

\cite{tao} further proposed an elimination algorithm using a novel estimator of the regression parameter based on a G-optimal design, with an improved sample complexity in $\sum_{i=1}^{h}\Delta_{(i)}^{-2}\log(1/\delta)$ where $\Delta_{(1)}\leq\dots\leq \Delta_{(h)}$ are the $h$ smallest gaps. This bound improves upon the complexity of the unstructured setting when $K \gg\dimfeat$. Some algorithms even match the minimal sample complexity either in the asymptotic regime $\delta\rightarrow 0$ \citep{degenne20gamification,jedra} or within multiplicative factors \cite{tanner}. Some adaptive algorithms, such as LinGapE \cite{xu} are also very effective in practice but without provably improving over unstructured algorithms in all instances.

The fixed-budget setting has been studied by~\citet{azizi, yang}, who propose algorithms based on Sequential~Halving \citep{karnin_almost_2013} where in each round, the active arms are sampled according to a G-optimal design. The best guarantees are those obtained by \cite{yang} who show that a budget $T$ of order $\log_2(h)\sum_{i=1}^{h}\Delta_{(i)}^{-2} \log(1/\delta)$ is sufficient to get an error smaller than $\delta$.
\cite{katz_samuels_empirical_process} propose an elimination algorithm that can be instantiated both in the fixed confidence and fixed budget settings and is close in spirit to our algorithm. However, unlike prior work, their optimal design aims at minimizing a new complexity measure called the Gaussian width that may better characterize the non asymptotic regime of the error. Extending this notion, or that of minimal (asymptotic) sample complexity to linear PSI is challenging due to the complex structure of the set of alternative models with a different Pareto set. In this work, our focus is on obtaining refined gap-based guarantees for the structured PSI problem.   
 
When $d>1$, the PSI identification problem has been mostly studied in the unstructured setting ($h=K$, canonical basis features). \citet{auer_pareto_2016} introduced some appropriate (non-trivial) notions of sub-optimality gaps for the PSI problem, which we recall in the next section. They proposed an elimination-based fixed-confidence algorithm whose sample complexity scales in $\sum_{i=1}^{K}\Delta_i^{-2}\log(1/\delta)$, which is proved to be near-optimal. A fully sequential algorithm with some slightly smaller bound was later given by \cite{kone2023adaptive}, who can further address different relaxations of the PSI problem. \cite{kone2023bandit} proposed the first fixed-budget PSI algorithm: a generic round-based elimination algorithm that estimates the sub-optimality gaps of \cite{auer_pareto_2016} and discard and classify some arms at the end of each round, with a sample complexity in  $\sum_{i=1}^{K}\Delta_i^{-2}\log(K)\log(1/\delta)$.

The multi-output linear setting that we consider in this paper was first studied by \cite{lu_multi_objective_generalized} from the Pareto regret minimization perspective. 
This model may also be viewed as a special case of the multi-output kernel regression model considered by \cite{zuluaga_e-pal_2016} when a linear kernel is chosen. This work provides guarantees for approximate identification of the Pareto set, scaling with the information gain. Choosing appropriately the approximation parameter in $\varepsilon$-PAL as a function of the smallest gap $\Delta_{\min}$ yields a fixed-confidence PSI algorithm with sample complexity of order $(h^2/\Delta_{\min}^2)\log(1/\delta)$. More recently, the preliminary work of \cite{kim2023pareto} proposed an extension of the fixed-confidence algorithm of \cite{auer_pareto_2016} with a robust estimator to simultaneously minimize the Pareto regret and identify the Pareto set. Their claimed sample complexity bound is in $(h/\Delta_{\min}^2)\log(1/\delta)$. 
\paragraph{Contributions}
We propose \gege{}, the first algorithm for PSI that relies on an optimal design to estimate the PSI gaps. In the fixed-confidence setting, \gege{} only uses $O(\log(1/\Delta_{(1)}))$ adaptive rounds to identify the
Pareto set, and we prove an improved sample complexity bounds in which $(h/\Delta_{\min}^2)$ is replaced by the sum $\sum_{i=1}^{h}\Delta_{(i)}^{-2}$.
Moreover, to the best of our knowledge, the fixed-budget variant of \gege{} is the first algorithm for fixed-budget PSI in a multi-output linear bandit model and enjoys near-optimal performance. Our experiments confirm these good theoretical properties and illustrate the impact of the structural assumption.
\section{SETTING}\label{sec:setting}
We formalize the linear PSI problem. Let $d,h\in \bN^\star$ and $h\leq K$. $\nu_1,\dots, \nu_K$ are distributions over $\bR^d$ with means (resp.) $\vmu_1, \dots, \vmu_K \in \bR^d$. We assume there are known feature vectors $x_1,\dots, x_K \in \bR^h$ associated to each arm and an unknown matrix $\Theta\in \bR^{h\times d}$ such that for any arm $k$, $\vmu_k = \Theta^\T x_k $. Let $\cX := (x_1 \dots x_K)^\T$ and $[K]= \{1,\dots, K\}$. The Pareto set is defined as 
$\cS^\star = \{ i \in [K] : \nexists j \in [K]\backslash \{i\} : \mu_i \preceq \mu_j\}$
in the sense of the following (Pareto) dominance relationship. 

\begin{definition}
For any two arms $i,j \in [K]$, $i$ is weakly 
dominated by $j$ if for any $c\in \{1, \dots, \dimvec\}$, $\mu_{i}(c) \leq \mu_{j}(c)$. An arm $i$ is 
dominated by $j$ ($\vmu_i \preceq \vmu_j$ or simply $i\preceq j$) if $i$ is weakly 
dominated by $j$ and there exists $c \in \{1, \dots, \dimvec \}$ such that $\mu_i(c) < \mu_j(c)$. An arm $i$ is strictly 
dominated by $j$ ($\vmu_i\prec \vmu_j$ or simply $i \prec j$) if for any $c \in \{1, \dots, \dimvec\}$, $\mu_i(c) < \mu_j(c)$.  
\end{definition}

In each round $t$, an agent chooses an action $a_t$ from $[K]$ and observes a response $y_t = \Theta^\T x_{a_t} + \eta_t $ where $(\eta_s)_{s\leq t}$ are\iid centered vectors in $\bR^d$ whose marginal distributions are $\sigma$-subgaussian.\footnote{A centered random variable $X$ is $\sigma$- subgaussian if for any $\lambda\in \bR, \log\bE[\exp(\lambda X)] \leq \lambda^2\sigma^2/2$.} In this stochastic game, the goal of the agent is to identify the Pareto set $\cS^\star$. In the fixed-confidence setting, given $\delta \in (0,1)$, the agent collects samples up to a (random) stopping time $\tau$ and outputs a guess $\widehat S_\tau$ that should satisfy $\bP(\widehat S_\tau \neq \cS^\star)\leq \delta$ while minimizing $\tau$ (either with high-probability or in expectation). In the fixed-budget setting, the agent should output a set $\widehat S_T$ after $T$ (fixed) rounds and minimize $e_T:= \bP(\widehat S_T \neq \cS^\star)$. 

\paragraph{Notation} The following notation is used throughout the paper.
 $\boldsymbol{\Delta}_n$ is the probability simplex of $\bR^n$ and if $A\in \bR^{n\times n}$ is positive semidefinite, for $x\in \bR^n$, $\|x\|_{A}^2 = x^\T A x$ and $x(i)$ denotes the $i$-th component of $x$. For $a,b\in \bR$, $a\land b := \min(a, b)$, and $(a)_+ := \max(a,0)$.

\subsection{Complexity Measures for Pareto Set Identification} 
\label{subsec:complexity_measure}
Choosing the features vectors to be the canonical basis of $\bR^K$ and $\Theta = (\mu_1,\dots,\mu_K)^\T$, we recover the unstructured multi-dimensional bandit model, in which the complexity of Pareto set identification is known to depend on some notion of sub-optimality gaps, first introduced by \citet{auer_pareto_2016}. These gaps can be expressed with the quantities
\[\m(i,j) := \min_{c \in [d]}\left[\mu_j({c}) - \mu_{i}(c)\right] \ \text{and} \  \M(i,j) := -\m(i,j) 
.\] We can observe that $\m(i,j) >0$ iff $i\prec j$ and represents the amount by which $j$ dominates $i$ when positive. Similarly, $\M(i,j)>0$ iff $i\npreceq j$ and when positive represents the quantity that should be added component-wise to $j$ for it to dominate $i$. The sub-optimality gap $\Delta_i$ measures the difficulty of classifying arm $i$ as optimal or sub-optimal and can be written (Lemma 1 of \citet{kone2023bandit})
\begin{equation}
\label{eq:def-gap}
	\Delta_i := \left\{\begin{array}{ll}
	               \Delta_i^\star := \max_{j\in [K]}\m(i, j) & \text{ if } i\notin \cS^\star \\
	                \delta_i^\star & \text{ else, }     \\
	                    \end{array}\right.
\end{equation}		
where $\delta_i^\star := \min_{j\neq i} [\M(i, j)\land (\M(j, i)_+ +(\Delta_j^\star)_+)]$. For a sub-optimal arm $i$, $\Delta_i$ is the smallest quantity by which $\vmu_i$ should be increased to make $i$ non-dominated. For an optimal arm $i$,  $\Delta_i$ is the minimum between some notion of "distance'' to the other optimal arms, $\min_{j\in \cS^\star\backslash\{i\}} [\M(i,j) \land \M(j,i)]$ and the smallest margin to the sub-optimal arms $\min_{j\notin \cS^\star} [\M(j,i)_+ + (\Delta_j^\star)_+]$. These quantities are illustrated in Appendix~\ref{sec:lower_bounds}. 

We assume without loss of generality that $\Delta_1\leq \dots\leq \Delta_{K}$ and we recall the quantities 
$H_1 = \sum_{i=1}^K {\Delta_i^{-2}}$ and $\quad H_2 := \max_{i\in [K]} {i}{\Delta_i^{-2}}$ which have been used to measure the difficulty of Pareto set identification respectively in fixed-confidence \citep{auer_pareto_2016} and fixed-budget \citep{kone2023bandit} settings. In this work, we introduce two analog quantities for linear PSI, namely 
\begin{equation}
\label{eq:eq-jj1}
H_{1, \text{lin}} = \sum_{i=1}^\dimfeat \frac{1}{\Delta_i^2} \quad \text{and} \quad H_{2, \text{lin}} := \max_{i\in [\dimfeat]} \frac{i}{\Delta_i^2}	
\end{equation}
and we will show that the hardness of linear PSI can be characterized by $H_{1, \text{lin}}$ and $H_{2, \text{lin}}$ respectively in the fixed-confidence and fixed-budget regimes. These complexity measures are smaller than $H_1$ and $H_2$, respectively, as they only feature the $h$ smallest gaps.
 To obtain this reduction in complexity, it is crucial to estimate the underlying parameter $\Theta \in \bR^{h\times d}$ instead of the $K$ mean vectors. 
 \subsection{Least Square Estimation and Optimal Designs}
 \label{sec:ls}
 Given $n$ arm choices in the model, $a_1,\dots,a_n$, we define $X_n := ({x_{a_1}} \dots {x_{a_{n}}})^\T \in \bR^{n\times h}$ and we denote by $Y_n := (y_1 \dots{y_{n}})^\T \in \bR^{n\times d}$ the matrix gathering the vector of responses collected. We define the information matrix as $V_n := X_n^\T X_n =   \sum_{i=1}^{K} {T_n(i)} { x_{i}} x_{i}^\T \in \bR^{h \times h}$ where ${T_i(n)}$ denotes the number of observations from arm $i$ among the $n$ samples. More generally, given $\omega \in \bR^K$, we define $V^{\omega} := \sum_{i=1}^K \omega(i) x_i x_i^\T$.  
 
 The multi-output regression model can be written in matrix form as $Y_n = X_n\Theta + H_n$ where $H_n =(\eta_1 \dots {\eta_{n}})^\T$ is the noise matrix. The least-square estimate $\widehat{\Theta}_n$ of the matrix $\Theta$ is defined as the matrix minimizing the least-square error $\text{Err}_n(A) := \left\|X_n  A  -  Y_n \right\|_{\text{F}}^2$. Computing the gradient of the loss yields $V_n \widehat{\Theta}_n = X_n^\top Y_n$.
If the matrix $V_n$ is non-singular, the least-square estimator can be written
 \[\widehat{\Theta}_n =  V_{n}^{-1} X_n^\T  Y_n.\]
In the course of our elimination algorithm, we will compute least-square estimates based on observation from a restricted number of arms, and we will face the case in which $V_n$ is singular. In this case, different choices have been made in prior work on linear bandits: \cite{alieva_robust} defines a custom ``pseudo-inverse'' while \cite{yang} define new contexts $\widetilde x_i $ that are projections of the $x_i$ onto a sub-space of dimension $\text{rank}(\cX_S)$ where $\cX_S:= (x_i: i\in S)^\T$ and $S$ is the set of arms that are active. We adopt an approach close to the latter, which is described below. Let the singular-value decomposition of $(\cX_S)^\T$ be $USV^\T$ where $U, V$ are orthogonal matrices and $B:= (u_1, \dots, u_m)$ is formed with the first $m$ columns of $U$ where $m = \rank{\cX_S}$. We then define 
\begin{equation}
\label{eq:defVN}
V_n^{\dagger} := B (B^\T V_nB)^{-1}B^\T 	\quad \text{and}\quad \widehat{\Theta}_n = V_n^{\dagger}X_n^\T Y_n.
\end{equation}
The following result addresses the statistical uncertainty of this estimator. 

\begin{restatable}{lemma}{lemPseudoInv}
\label{lem:lem-pseudo-inv}
If the noise $\eta_t$ has covariance $\Sigma \in \bR^{d\times d}$ and $a_1,\dots, a_{n}$ are deterministically chosen then for any $x_i \in \{x_{a_1},\dots,x_{a_n}\}$, $\emph{Cov}(\widehat{\Theta}_n ^\T x_i) =   \| x_i\|_{ V_n^{\dagger}}^2\Sigma.$
\end{restatable}
Therefore, estimating all arms'mean uniformly efficiently amounts to pull $\{a_1,\dots,a_n\}$ to minimize $\max_{i\in S} \| x_i\|_{ V_n^{\dagger}}^2$. The continuous relaxation of this problem is equivalent to computing an allocation 
\begin{equation}
\label{eq:designg}
\omega^\star_S \in \argmin_{\omega \in \boldsymbol{\Delta}_{\lvert S \rvert}} \max_{i \in S} \|\widetilde x_i\|_{(\widetilde V^{{w}})^{-1}}^2 	
\end{equation}
where $\widetilde x_i := B^\T x_i$, $\widetilde V^\omega := \sum_{i\in S} \omega(s_i) \widetilde x_i \widetilde x_i^\T $ and $i\mapsto s_i$ maps $S$ to $\{1,\dots, \lvert S \rvert\}$. \eqref{eq:designg} is a G-optimal design over the features $(B^\T x_i, i\in S)$ and it can be interpreted as a distribution over $S$ that yields a uniform estimation of the mean responses for \eqref{eq:defVN}. This is formalized in Appendix~\ref{sec:compte_round}.  
\section{OPTIMAL DESIGN ALGORITHMS FOR LINEAR PSI}
\label{sec:algo}
Our elimination algorithms operate in rounds. They progressively eliminate a portion of arms and classify them as optimal or sub-optimal based on empirical estimation of their gaps. In each round, a sampling budget is allocated among the surviving arms based on a G-optimal design. 

\subsection{Optimal Designs and Gap Estimation} 

At round $r$, we denote by $A_r$ the set of arms that are still active. To estimate the means and, henceforth, the gaps, we first compute an estimate of the regression matrix denoted $\widehat{\Theta}_r$. This estimate is obtained by carefully sampling the arms using the integral rounding of a G-optimal design. 

\begin{algorithm}
\caption{\text{OptEstimator$(S, N, \kappa)$}}
   \label{alg:est_gaps}

   {\bfseries Input:} {$S \subset [K]$, sample size $N$, precision $\kappa$}
   
   {Compute}  the transformed features $\widetilde{\cX}_S = (B^\T x_{i}, i \in S)$ with $B$ as defined in Section~\ref{sec:ls}
   
    {Compute a G-optimal design $w^\star_S$ over the set $\widetilde{\cX}_S$}
   
  {Pull} $(a_1,\dots, a_N) \gets \texttt{ROUND}(N, \widetilde{\cX}_S, \omega^\star_S, \kappa)$  and collect responses $y_1, \dots, y_N$ 

   {Compute} $V_N^{\dagger}$ as in \myeqref{eq:defVN} and compute the OLS estimator on the samples collected
   $$\widehat\Theta \gets  V_N^{\dagger} \sum_{t=1}^N x_{a_t}^\T y_t $$ 
 
 {\bfseries return:} {$\widehat \Theta$}

\end{algorithm}

Algorithm~\ref{alg:est_gaps} takes as input a set of arms $S$, a budget $N$ and chooses some $N$ arms to pull (with repetitions) based on an integer rounding of $w_S^\star$, a continuous G-optimal design over the set $\{\widetilde{x}_i, i\in S\}$ of (transformed) features associated to that arms. Several rounding procedures have been proposed in the literature, and we use that of \cite{zhu}, henceforth referred to as \texttt{ROUND}. In Appendix~\ref{sec:compte_round}, we show that $\texttt{ROUND}(N,\widetilde{\cX}_S,w_S^\star, \kappa)$ outputs a sequence of arms $a_1,\dots,a_N \in S$  such that $\max_{i \in S} ||x_i||_{V_N^\dagger}^2 \leq (1+6\kappa)\tfrac{F_S(w_S^\star)}{N}$, where $F_S(w_S^\star)$ is the optimal value of \eqref{eq:designg}. Using the Kiefer-Wolfowitz theorem \citep{kiefer_wolfowitz_1960}, we further prove that $F_S(w_S^\star) = {h}_S$, the dimension of $\text{span}(\{x_i, i\in S\})$. This observation is crucial to prove the following concentration result at the heart of our analysis. 
 
\begin{restatable}{lemma}{lemConcentrDesign}
\label{lem:lem-concentr-design}
Let $S\subset [K]$, $\kappa \in (0, 1/3]$ and $N\geq 5\dimfeat_S/\kappa^2$ where 
$\dimfeat_S = \emph{dim(span($\{x_i: i\in S\}$))}$. The output $\widehat\Theta$ of \optDesign($S$, $N$, $\kappa$) satisfies for all $\eps>0$ and $i\in S$
	$$ \bP\left(\|(\Theta - \widehat\Theta)^\T x_i\|_\infty \geq \veps\right) \leq 2 \dimvec \exp\left( - \frac{N\eps^2}{2(1+6\kappa)\sigma^2 \dimfeat_S}\right).$$
\end{restatable}

Once the parameter $\widehat{\Theta}_r$ has been obtained as an output of Algorithm~\ref{alg:est_gaps} with $S=A_r$ and an appropriate value of the budget $N$, we compute estimates of the mean vectors as
 $\vmuh_{i,r} := \widehat\Theta_r^\T x_i$
and the empirical Pareto set of active arms, \[S_r := \{i \in A_r : \nexists j \in A_r : \muh_{i,r} \prec \muh_{j,r} \}.\] 
 In both the fixed-confidence and fixed-budget settings, at round $r$, after collecting new samples from the surviving arms, \gege{} discards a fraction of the arms based on the empirical estimation of their gaps. We first introduce the empirical quantities used to compute the gaps: 
 \begin{eqnarray*}
 	\Mh(i,j; r) &:=& \max_{c\in [\dimvec]} [\muh_{i,r}(c) - \muh_{j,r}(c)]  \quad \text{and} \\ \mh(i,j; r) &:=& \min_{c \in [\dimvec]}[\muh_{j,r}(c) -  \muh_{i,r}(c)].
 \end{eqnarray*}
We define for any arm $i \in A_r$,  the empirical estimates of the PSI gaps  as: 
\begin{equation}
\label{eq:eq-def-gap}
\widehat \Delta_{i,r} := 
\begin{cases}
\widehat\Delta_{i, r}^\star :=  \max_{j\in A_r}\mh(i,j;r) &\text{ if } i\in A_r\backslash S_r
\\ \widehat \delta_{i,r}^\star  &\text{ if } i \in S_r
	\end{cases}
\end{equation}
with $\widehat \delta_{i,r}^\star := \min_{j \in A_r \backslash\{i\}} [\Mh(i,j;r) \!\land \! (\Mh(j,i;r)_+ + (\widehat \Delta_{i, r}^\star)_+) ]$; the empirical estimates of the gaps introduced earlier in Section~\ref{subsec:complexity_measure}. 
Differently from BAI, as the size of the Pareto set is unknown, we need an accept/reject mechanism to classify any discarded arm. This mechanism is described in detail in the next sections for the fixed-budget and fixed-confidence versions. 

\paragraph{Final output} In both cases, letting $A_r$ be the set of active arms and $B_r$ be the set of arms already classified as optimal at the beginning of round $r$, \gege{} outputs $B_{\tau+1} \cup A_{\tau+1}$ as the candidate Pareto optimal set, where $\tau$ denotes the final round. And $A_{\tau+1}$ contains at most one arm. 
\subsection{Fixed-budget algorithm}
 
Algorithm~\ref{alg:gegefb}, operates over $\lceil \log_2(\dimfeat)\rceil$ rounds, with an equal budget of $T/\lceil \log_2(\dimfeat)\rceil$ allocated per round. By construction $\lvert A_{\lceil \log_2(\dimfeat)\rceil+1} \lvert = 1$. 
At the end of round $r$, the $\lceil \dimfeat/2^r\rceil$ arms with the smallest empirical gaps are kept active while the remaining arms are discarded and classified as Pareto optimal (added to $B_{r+1}$) if they are empirically optimal (belonging to set $S_r$) and deemed sub-optimal otherwise. 
If a tie occurs, we break it to eliminate arms that are empirically sub-optimal. This is crucial to prove the guarantees on the algorithm, as sketched in Section~\ref{sec:analysis}. 
  
  \begin{algorithm}
\caption{GEGE: G-optimal Empirical Gap Elimination [\textcolor{red}{fixed-budget}]}
   \label{alg:gegefb}

 {\bfseries Input:} {budget $T$}
 
 {\bfseries Initialize:} {let $A_1 \gets [K], B_1 \gets \emptyset$, $D_1 \gets \emptyset$ }
 
   \For{$r=1$ {\bfseries to} $\lceil \log_2(\dimfeat)\rceil$}{
      
     {Compute} {$\widehat\Theta_r  \gets \text{\optDesign$(A_r, T/\log_2(\dimfeat), 1/3)$}$}
      
   {Compute} $S_r$ the empirical Pareto set  and the empirical gaps $\widehat\Delta_{i,r}$ with Eq.\eqref{eq:eq-def-gap}
   
    {Compute} $A_{r+1}$ the set of $\left\lceil \frac{\dimfeat}{2^r}\right\rceil$ arms in $A_{r}$ with the smallest empirical gaps 
    \tcp{ties broken by keeping arms of $S_r$}
    Update $B_{r+1} \gets B_r \cup \left\{ S_r \cap (A_r \backslash A_{r+1})\right\}$ and $D_{r+1} \gets D_r \cup \left\{  (A_r \backslash A_{r+1}) \backslash S_r \right\}$
   }
    {\bfseries return:} {$B_{\lceil \log_2(\dimfeat) \rceil +1 } \bigcup A_{\lceil \log_2(\dimfeat) \rceil +1 }$}
\end{algorithm}
\begin{restatable}{theorem}{thmProbGege}
\label{thm:thm-prob-gege}
The probability of error of Algorithm~\ref{alg:gegefb} run with budget $T\geq 45\dimfeat\log_2\dimfeat$ is at most 
$$\exp\left(-\frac{T}{1200\sigma^2H_{2, \text{lin}}\lceil \log_2\dimfeat\rceil}+\log C(\dimfeat, \dimvec, K)\right)$$
where $C(\dimfeat, \dimvec, K) = 2d\left({K} + {\dimfeat} + {\lceil \log_2\dimfeat\rceil }\right)$. 
\end{restatable}

To the best of our knowledge, \gegeFb{} is the first algorithm with theoretical guarantees for fixed-budget linear PSI. Our result shows that in this setting, the probability of error scales only with the first $\dimfeat$ gaps. \citet{kone2023bandit} proposed EGE-SH, an algorithm for fixed-budget PSI in the unstructured setting whose probability of error is essentially upper-bounded by  
$$ \exp\left(- \frac{T}{288\sigma^2 H_2\log_2K }+ \log(2\dimvec(K-1)\lvert \cS^\star \lvert\log_2K) \right).$$
 Therefore, \gegeFb{}  largely improves upon EGE-SH when $K\gg \dimfeat$. 
Moreover, when $K=\dimfeat$ and $x_1,\dots, x_K$ is the canonical $\bR^\dimfeat$-basis, both algorithms coincide,
thus, \gegeFb{} can be seen as a generalization of EGE-SH.

We state below a lower bound for linear PSI in the fixed-budget setting, showing that \gegeFb{} is optimal in the worst case, up to constants and a $\log_2(h)$ factor. 
\begin{restatable}{theorem}{thmLbdFb}
 \label{thm:thm-lbd-fb} 
Let $\mathbb{W}_H$ be the set of instances with complexity $H_{2, \text{lin}}$ smaller than $H$.
For any budget $T$, letting $\widehat{S}_T^\cA$ be the output of an algorithm $\cA$, it holds that 
	$$\inf_{\cA} \sup_{\nu \in \mathbb{W}_H} \bP_\nu(\widehat{S}_T^\cA \neq \cS^\star(\nu))\geq \frac{1}{4}\exp\left( - \frac{2T}{H\sigma^2}\right).$$  
\end{restatable}
\subsection{Fixed-confidence algorithm}
In round $r$, Algorithm~\ref{alg:gegeFc}, allocates a budget $t_r$ to compute an estimator $\widehat \Theta_r$ of $\Theta^\star$ by calling \myalg{alg:est_gaps}. $t_r$ is computed so that through $\widehat \Theta_r$, the mean of each arm is estimated with precision $\veps_r/4$ with probability larger than $1-\delta_r$ (using Lemma~\ref{lem:lem-concentr-design}). Then, the empirical Pareto set $S_r$ of the active arms is computed, and the empirical gaps are updated following \eqref{eq:eq-def-gap}. 
 
At the end of round $r$, empirically optimal arms (those in $S_r$) whose empirical gap is larger than $\varepsilon_r$ are discarded and classified as optimal (added to $B_{r+1}$). Empirically sub-optimal arms whose empirical gap is larger than $\varepsilon_r/2$ are also discarded and classified as sub-optimal (added to $D_{r+1}$). Just like its fixed-budget version, the algorithm tends to favor elimination of seemingly sub-optimal arms. 

\begin{algorithm}
\caption{GEGE: G-optimal Empirical Gap Elimination [\textcolor{red}{fixed-confidence}]}
   \label{alg:gegeFc}
 {\bfseries Initialize:} {$A_1\gets [K]$, $B_1 \gets \emptyset$, $D_1 \gets \emptyset$, $r\gets 1$}
 
   \While{$\lvert A_r\lvert>1$}{
  
  {Let} {$\veps_r \gets 1/(2\cdot 2^r)$ and $\delta_r \gets 6\delta/\pi^2 r^2$ and $h_r \gets \text{dim(span($\{x_i: i\in A_r\}$))}$}

   {Update}{ $t_r := \left\lceil \frac{32(1+3\veps_r)\sigma^2 h_r}{\eps_r^2} \log(\frac{\lvert A_r\lvert \dimvec}{2\delta_r})\right\rceil $}
   
   {Compute} {$\widehat\Theta_r  \gets \text{\optDesign$(A_r, t_r,  \veps_r)$}$}

   {Compute} $S_r$ and the empirical gaps $\widehat\Delta_{i,r}$ with \myeqref{eq:eq-def-gap}
   
    {Update} $B_{r+1} \gets B_r \cup \{ i \in S_r: \widehat{\Delta}_{i, r} \geq  \veps_r \}$ and $D_{r+1} \gets D_r \cup \{ i \in A_r\backslash S_r: \widehat{\Delta}_{i, r} \geq \veps_r/2\}$

  {Update} $A_{r+1} \gets A_r\backslash\left(D_{r+1} \cup B_{r+1}\right)$
   
   $r\gets r+1$ 
   }
   {\bfseries return:} {$B_{r} \cup A_{r}$}
\end{algorithm}

\begin{restatable}{theorem}{thmGegeFcSc}
\label{thm:thm-gege-fc-sc} 
The following statement holds with probability at least $1-\delta$: 
	Algorithm~\ref{alg:gegeFc} identifies the Pareto set using at most $$ \log_2(2/\Delta_1) + O \lp \sum_{i=2}^\dimfeat \frac{\sigma^2}{\Delta_{i}^2} \log\left(\frac{K\dimvec}{\delta} \log_2\left(\frac2{\Delta_{i}}\right)\right) \rp$$
	samples and $\lceil \log_2(1/\Delta_1)\rceil$  rounds, 
	where $O(\cdot)$ hides universal multiplicative constant (explicit in Appendix~\ref{sec:full_proof}).   
\end{restatable} 
This result shows that the complexity of \myalg{alg:gegeFc} scales only with the first $\dimfeat$ gaps. 
In particular, when $K\gg \dimfeat$ using our algorithm substantially reduces the sample complexity of PSI. 

In Table~\ref{tab:tab-comp}, we compare the sample complexity of \gege{} to that of existing fixed-confidence PSI algorithms, showing that \gege{} enjoys stronger guarantees than its competitors. We emphasize that both \cite{kim2023pareto} and \cite{zuluaga_e-pal_2016} use uniform sampling and do not exploit an optimal design, which prevents them from reaching the guarantees given in Theorem~\ref{thm:thm-gege-fc-sc}. 
\begin{table}[htb]
\caption{Sample complexity up to constant multiplicative terms of different algorithms for PSI in the fixed-confidence setting.}
\label{tab:tab-comp}
\centering
\begin{tabular}{|p{3.cm}|p{3.2cm}|p{0.74cm}|} 
  \hline 
  Algorithm & {Upper-bound on $\tau_\delta$}& Linear PSI \\  

\hline \hline 
$\!\!$\cite{zuluaga_e-pal_2016} $\!\!\!\!$ & $\!\!\left(\frac{h^2}{\Delta_{\min}^2}\right)\!\log^3\left(\frac{Kd}{\delta}\right)$  & \color{OliveGreen}{\ding{52}} \\ 
 \hline 

 $\!\!$\cite{kone2023adaptive} & $\!\!\sum_{i=1}^K\frac{1}{\Delta_i^2}\!\log(\frac{K\dimvec}{\delta}\!\log(\frac{1}{\Delta_i})\!)$ & \color{BrickRed}{\ding{56}}  \\
 \hline 
 $\!\!$\cite{kim2023pareto}  & $\!\!\frac{h}{\Delta_{\text{min}}^2}\!\log(\frac{d(h\vee K)}{\delta\Delta_\text{min}^2})$ & \color{OliveGreen}{\ding{52}} \\ 
 \hline 
 $\!\!$\gege{} ({Ours}) &  $\!\!\sum_{i=1}^\dimfeat \frac{1}{\Delta_{i}^2}\! \log(\frac{K\dimvec}{\delta} \!\log(\frac1{\Delta_{i}})\!)$&\color{OliveGreen}{\ding{52}} \\\hline  \end{tabular}
\end{table}

We state a lower bound, showing that our algorithm is essentially minimax optimal for linear PSI. 

\begin{restatable}{theorem}{thmLbdFc}
\label{thm:thm-lbd-fc}
For any $K, \dimvec, \dimfeat \in \bN$, there exists a set $\cB(K,\dimvec, \dimfeat)$ of linear PSI instances s.t for $\nu\in \cB(K,\dimvec, \dimfeat)$ and for any $\delta$-correct algorithm for PSI, with probability at least $1-\delta$, 
	$$\tau_\delta^\cA \geq \Omega\left(H_{1, \text{lin}}(\nu) \log(\delta^{-1})\right).$$
\end{restatable}

\begin{remark}
When $K=\dimfeat$ and $x_1,\dots, x_K$ form the canonical $\bR^\dimfeat$ basis, we recover the classical PSI problem. We note that, unlike its fixed-budget version, \gege{} does not coincide with an existing PSI identification algorithm. Instead, it matches the optimal guarantees of \cite{kone2023adaptive} while needing only $\lceil \log(1/\Delta_1)\rceil$ rounds of adaptivity, which is the first fixed-confidence PSI algorithm having this property. Such a batched algorithm may be desirable in some applications, e.g., in clinical trials where measuring different biological indicators of efficacy can take time. 
\end{remark}

\paragraph{\gegeFc{} for $\veps$-PSI} Algorithm~\ref{alg:gegeFc} can be easily modified to identify an $\varepsilon$-Pareto Set. As introduced in \citet{auer_pareto_2016}, an $\veps$-Pareto Set $S_\veps$ is such that $S^\star \C S_\veps $ and for any arm $i \in S_\veps$, $\Delta_i^\star \leq \veps$: it contains the Pareto Set and possibly some sub-optimal arms that are $(\veps)$-close to be optimal. Such a relaxation is particularly useful in instances with small gaps or when identifying the exact Pareto Set may be unnecessarily restrictive. To identify an $\veps$-Pareto Set, we relax the stopping condition: instead of stopping when it remains only one active arm (i.e., $|A_r| \leqslant 1$), we stop when $(|A_r|\leqslant 1\; \text{ or } \; \varepsilon_r \leqslant \varepsilon/4 $) holds. After stopping, the same set is recommended, namely $A_\tau \cup B_\tau$. The guarantees of \gegeFc{} under this modification are discussed in Section~\ref{ssec:varpsi}.  
\section{UNFIED ANALYSIS OF \gege}
\label{sec:analysis}

Before sketching our proof strategy, we highlight a key property of PSI that makes the analysis different from classical BAI settings. Let $a$ be a (Pareto) sub-optimal arm. From \eqref{eq:def-gap}, there exits $a^\star \in \cS^\star$ such that $\Delta_a = \m(a, a^\star)$ and importantly, $a^\star$ could be the unique arm dominating $a$. 
 Therefore, discarding $a^\star$ before $a$ may result in the latter appearing as optimal in the remaining rounds, thus leading to the misidentification of the Pareto set. 
 
 To avoid this, an elimination algorithm for PSI should guarantee that if a sub-optimal arm $a$ is active, then $a^\star$ is also active. 
 We introduce the following event 
$$ \cP_r := \{ \forall\; s \leq r: \forall i\in (\cS^\star)^\complement, \; i\in A_s \Rightarrow i^\star \in A_r \}.$$
An important aspect of our proofs is to control the occurrence of $\cP_\infty$ (by convention, if $\cP_t$ holds and $A_s = \emptyset$ for any $s\geq t$, then $\cP_\infty$ holds). 
The first step of the proof is to show that when $\cP_r$ holds, we can control the deviations of the empirical gaps, which is essential to guarantee the correctness of \gege{} and to control its sample complexity in fixed-confidence. We now define for $\eta>0$, the good event 
\begin{equation}
 \cE^r(\eta) = \left\{ \forall\;  i,j \in A_r: \|(\widehat\Theta_r - \Theta)^\T (x_i-x_j)\|_\infty \leq \eta \right\}.
\end{equation} 
Letting $n_r = |A_r|$ and $\lambda$ a constant to be specified, we introduce $\cE_{\text{fb}}^\lambda:= \cap_{r=1}^{\lceil\log_2(h)\rceil} \cE^{r}(\lambda \Delta_{n_{r+1}+1})$ and $\cE_{\text{fc}}:= \cap_{r=1}^\infty \cE^r(\veps_r/2)$. 
We then prove by concentration and induction the following key result. 

\begin{restatable}{proposition}{propFbFc}
\label{prop:propFbFc}
	Let $\lambda \in (0,1/5)$ and assume $\cE_{\mathrm{fc}}$  (resp. $\cE_{\mathrm{fb}}^\lambda$ in fixed-budget) holds. Then at any round $r$, $\cP_r$ holds and for all arm $i\in A_r$, 
	$$\widehat{\Delta}_{i,r} - \Delta_i \geq  \begin{cases}
		-\eta_r & \text{ if } i\in \cS^\star \\
		-\eta_r/2 & \text{ else,} 
		\end{cases} \quad \text{ where } 
			$$
			$$ \eta_r :=  \begin{cases}
		2\lambda \Delta_{n_{r+1}+1} & \text{(fixed-budget)}\\
		\veps_r & \text{(fixed-confidence).}
		\end{cases}$$
\end{restatable}
Building on this result, we show that the recommendation of
Algorithm~\ref{alg:gegefb} is correct on $\cE_{\mathrm{fb}}^{\lambda}$, so its probability of error 
is upper-bounded by $\inf_{\lambda\in(0, 1/5)} \bP(\cE_{\mathrm{fb}}^{\lambda})$. We conclude the proof of Theorem~\ref{thm:thm-prob-gege} by upper bounding this probability (see Appendix~\ref{sec:analysis_fb}).

Similarly, using Proposition~\ref{prop:propFbFc} we prove the correctness of \myalg{alg:gegeFc}
on $\cE_{\text{fc}}$:  at any round $r$, $B_r \subset \cS^\star$ and $D_r\subset (\cS^\star)^\complement$. 

To further upper bound its sample complexity, we need an additional result to control the size of $A_r$. 
\begin{restatable}{lemma}{lemCardActive} 
\label{lem:lem-card-active}
	The following statement holds for \myalg{alg:gegeFc} on the event $\cE_{\emph{fc}}$: for all $p\in [K]$, after $\lceil \log(1/\Delta_p)\rceil$ rounds it remains less than $p$ active arms. In particular, \gegeFc{} stops after at most $\lceil \log(1/\Delta_1)\rceil$ rounds. 
\end{restatable} 
The proof of this lemma is given in Appendix~\ref{sec:proof_lem_card}.  
To get the sample complexity bound of Theorem~\ref{thm:thm-gege-fc-sc}, some extra arguments are needed. We sketch some elements below (the full proof is given in Appendix~\ref{sec:full_proof}). Assume $\cE_{\text{fc}}$ holds and let $\tau_\delta$ be the sample complexity of \myalg{alg:gegeFc}. Lemma~\ref{lem:lem-card-active} yields $\tau_\delta \leq \sum_{r=1}^{\lceil \log(1/\Delta_1)\rceil} \Omega(h_r/\veps_r^2)$ with $h_r\leq \lvert A_r \lvert$. 

Using Lemma~\ref{lem:lem-card-active}, we introduce "checkpoints rounds'' between which we control $\lvert A_r \rvert$ and thus $h_r$. Let the sequence $(\alpha_s)_{s\geq 0}$ defined 
as $\alpha_0=0$ and $\alpha_s = \lceil \log_2(1/\Delta_{\lfloor \dimfeat/2^s\rfloor})\rceil$, for $s\geq 1$. 
Simple calculation yields $\alpha_{\lfloor\log_2(\dimfeat)\rfloor} = \lceil \log_2(1/\Delta_{1})\rceil$ and $ \{1,\dots, \lceil \log_2(1/\Delta_{1})\rceil\} = \cup_{s=1}^{\lfloor \log_2(\dimfeat)\rfloor} \llbracket 1+\alpha_{s-1}, \alpha_s \rrbracket.$ Therefore 
$$\tau_\delta \leq  \sum_{s=1}^{\lfloor \log_2(\dimfeat)\rfloor} \sum_{r=\alpha_{s-1}+1}^{\alpha_s} \Omega(\lvert A_r \lvert/\veps_r^2).$$ 

Now by Lemma~\ref{lem:lem-card-active}, for $r>\alpha_s$, $\lvert A_r \lvert\leq \lfloor \dimfeat/2^s\rfloor$, so essentially 
$\tau_\delta \leq \sum_{s=1}^{\lfloor \log_2(\dimfeat)\rfloor} \Omega(4^{\alpha_s}\lfloor \dimfeat/2^s\rfloor )$. 

Carefully re-indexing this sum and addressing a few more technicalities, we obtain the result in Theorem~\ref{thm:thm-gege-fc-sc}. Showing that $\bP(\cE_{\mathrm{fc}})\geq 1-\delta$, from Lemma~\ref{lem:lem-concentr-design} completes the proof. 

\section{EXPERIMENTS}
\label{sec:expe}

We evaluate GEGE in real-world and synthetic instances. In the fixed-budget setting, we compare against EGE-SH and EGE-SR \citep{kone2023bandit}, two algorithms for unstructured PSI in the fixed-budget setting, and a uniform sampling baseline.
 
In the fixed-confidence setting, we compare to APE \citep{kone2023adaptive}, a fully adaptive algorithm for unstructured PSI, and PAL \citep{zuluaga_active_201}, an algorithm that uses Gaussian process modeling for PSI, instantiated with a linear kernel. 

\subsection{Experimental protocol}
We describe below the datasets in our experiments, and we detail our experimental setup. 

\paragraph{Synthetic instances} 
We fix features $x_1, \dots, x_\dimfeat$ and $\Theta$ common to the instances described below. For any $K \geq h$ we define a linear PSI instance $\nu_K$ augmented with arms $x_{\dimfeat +1},\dots, x_K$ chosen so that arms $1,\dots, \dimfeat$ have the same lowest gaps in $\nu_K$. This implies that the complexity terms $H_{1, \text{lin}}$ and  $H_{2, \text{lin}}$ are constant for such instances, irrespective of the number of arms. We set $h=8, d=2$.

\paragraph{Real-world dataset} NoC \citep{noc} is a bi-objective optimization dataset for hardware design. The goal is to optimize $\dimvec=2$ performance criteria: energy consumption and runtime of the implementation of a Network on Chip (NoC). The dataset contains $K=259$ implementations, each of them described by $\dimfeat = 4$ features. 
 
In each instance, we report, for different algorithms,
the empirical error probability (fixed-budget) and the empirical distribution of the sample complexity (fixed-confidence) averaged over 500 seeded runs. We set $\delta=0.01$ for the fixed-confidence experiments and $T = H_{2,\text{lin}}$ for fixed-budget.

\begin{figure}[htb]
      \centering
      \begin{minipage}{0.46\linewidth}
              \includegraphics[width=\linewidth]{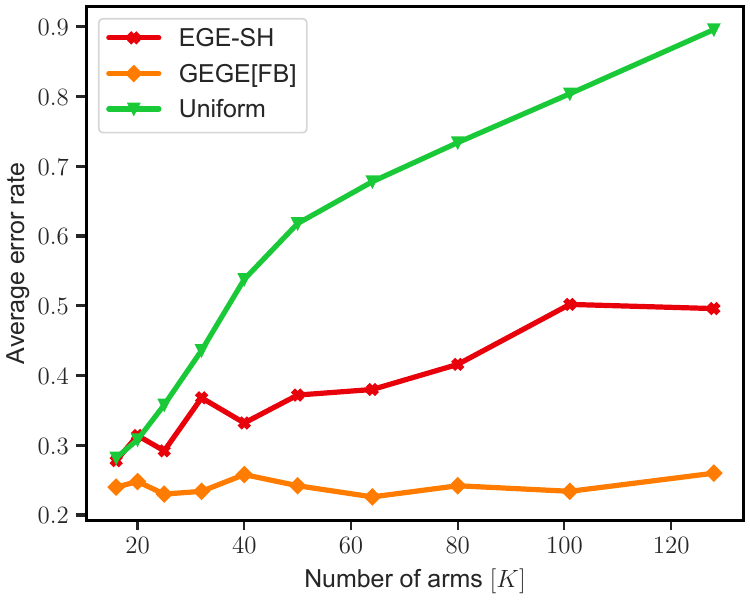}	
              \captionof{figure}{Average misidentification rate w.r.t $K$ on the synthetic dataset}
              \label{fig:fig-synth-fb}
      \end{minipage}
\hspace{0.1cm}
      \begin{minipage}{0.46\linewidth}
              \includegraphics[width=\linewidth]{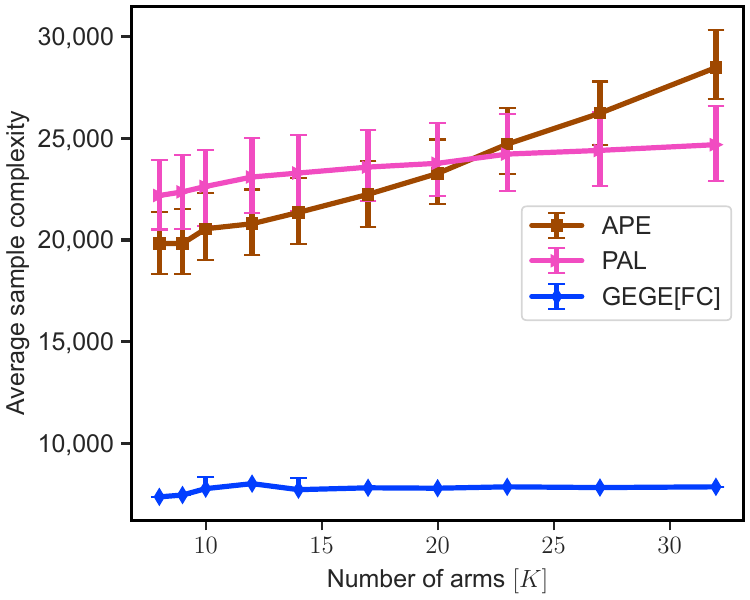}	
              \captionof{figure}{Average sample complexity w.r.t $K$ in the synthetic experiment}
              \label{fig:fig-synth-fc}
      \end{minipage}
      \end{figure}
      
\begin{figure}[H]
\centering
\hspace{0.1cm}
        \begin{minipage}{0.46\linewidth}
              \includegraphics[width=\linewidth]{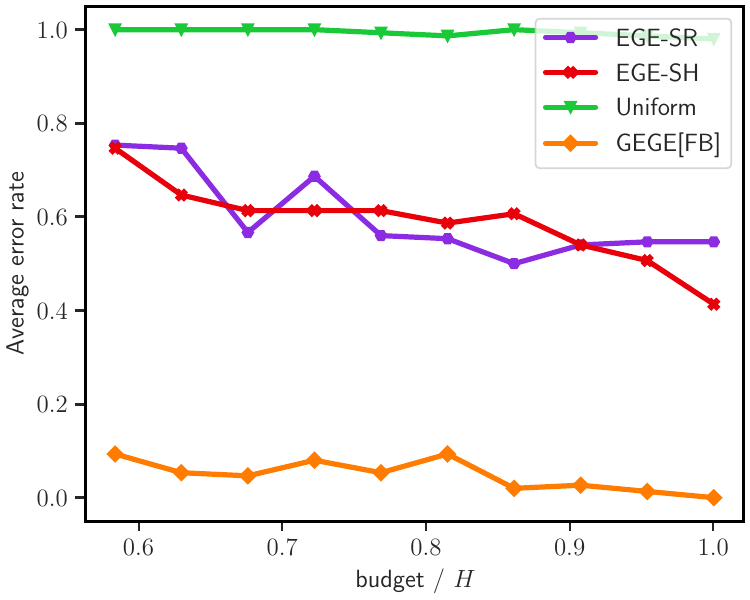}	
              \captionof{figure}{Average misidentification rate w.r.t $T$ on NoC experiment} 
              \label{fig:fig-noc-fb}
      \end{minipage}
\hspace{0.1cm}
            \begin{minipage}{0.46\linewidth}
              \includegraphics[width=\linewidth]{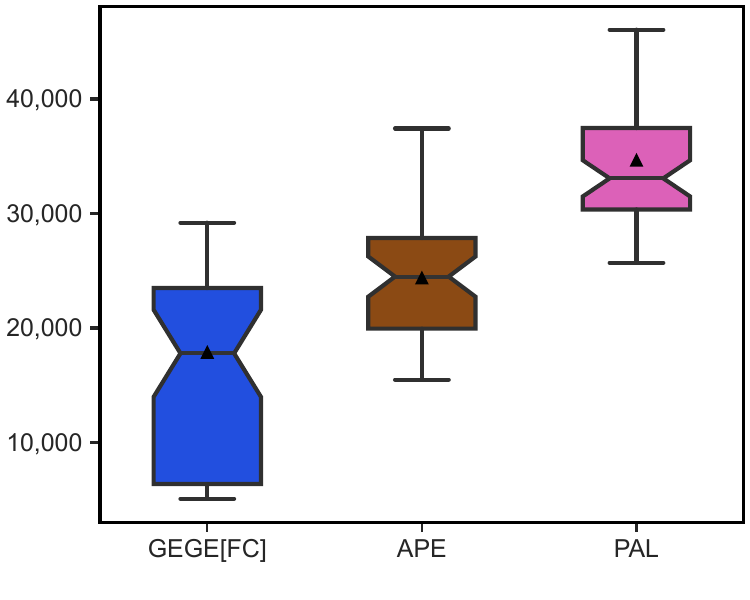}
              \captionof{figure}{Empirical sample complexity in the NoC experiment}
              \label{fig:fig-noc-fc}
      \end{minipage}
 \end{figure}   
 
\subsection{Summary of the results}

By Theorem~\ref{thm:thm-prob-gege} and~\ref{thm:thm-gege-fc-sc}, on the synthetic instance with $K$ arms, the sample complexity of GEGE should be a constant plus a $\log(K)$ term. This is coherent with what we observe: \myfig{fig:fig-synth-fb} shows that the probability of error of \gege{} merely increases with $K$, whereas for EGE-SH/SR, it grows much faster. Similarly, on \myfig{fig:fig-synth-fc}, the sample complexity of \gege{} does not significantly increase with $K$, unlike that of \algauername{}. Therefore, \gege{} only suffers a small cost for the number of arms. 

For the real-world scenario, GEGE significantly outperforms its competitors in both settings.
\myfig{fig:fig-noc-fc} shows that it uses significantly fewer samples to identify the Pareto set compared to both \algauername{} and \pal{}. \myfig{fig:fig-noc-fb} shows that the probability of misidentification of \gege{} is reduced by up to $0.5$ compared to EGE-SH. Moreover, it is worth noting that EGE-SH requires $T\geq K\log_2(K) \approx 2000$ (for NoC) to run on this instance while \gege{} only needs $T\geq \log_2(h)$.
 
We reported runtimes around 10 seconds for single runs on instances with up to $K=500, d=8$ (cf Table~\ref{table:runtime} in Appendix~\ref{sec:impl}). The time and memory complexity is addressed in Appendix~\ref{sec:impl}, and additional details about the implementation are provided. 
 
Appendix~\ref{sec:add_exp} contains additional experimental results on a real-world multi-criteria optimization problem with $K=768$ arms.   
\section{CONCLUSION AND FINAL REMARKS}
\label{sec:conclusion}

We have proposed the first algorithms for PSI in a multi-output linear bandit model that are guaranteed to outperform their unstructured counterparts. They leverage optimal design approaches to estimate the means vector and some sub-optimality gaps for PSI. In the fixed-budget setting, \gegeFb{} is the first algorithm with nearly optimal guarantees for linear PSI. In the fixed-confidence setting, \gegeFc{} provably outperforms its competitors both in theory and in our experiments. It is also the first fixed-confidence PSI algorithm using a limited number of batches. 

While the sample complexity of \gege{} features a complexity term depending only on $\dimfeat$ gaps, we still have $\log(K)$ terms due to union bounds. \cite{katz_samuels_empirical_process} showed that such union bounds can be avoided in linear BAI by using results from supremum of empirical processes. Further work could investigate if these observations would apply in linear PSI. In the alternative situation where $\dimfeat \gg K$, for example, in an RKHS, following the work of \cite{camilleri21_kernel}, we could investigate how to extend this optimal design approach to PSI with high-dimensional features. 

\subsubsection*{Acknowledgements}
Cyrille Kone is funded by an Inria/Inserm PhD grant. This work has been partially supported by the French National Research Agency (ANR) in the framework of the PEPR IA FOUNDRY project (ANR-23-PEIA-0003).

\bibliography{main.bib}

\begin{thebibliography}{39}
\providecommand{\natexlab}[1]{#1}
\providecommand{\url}[1]{\texttt{#1}}
\expandafter\ifx\csname urlstyle\endcsname\relax
  \providecommand{\doi}[1]{doi: #1}\else
  \providecommand{\doi}{doi: \begingroup \urlstyle{rm}\Url}\fi

\bibitem[Alieva et~al.(2021)Alieva, Cutkosky, and Das]{alieva_robust}
A.~Alieva, A.~Cutkosky, and A.~Das.
\newblock Robust pure exploration in linear bandits with limited budget.
\newblock In \emph{Proceedings of the 38th International Conference on Machine
  Learning}, Proceedings of Machine Learning Research. PMLR, 2021.

\bibitem[Allen-Zhu et~al.(2017)Allen-Zhu, Li, Singh, and Wang]{zhu}
Z.~Allen-Zhu, Y.~Li, A.~Singh, and Y.~Wang.
\newblock Near-optimal design of experiments via regret minimization.
\newblock In \emph{Proceedings of the 34th International Conference on Machine
  Learning}. JMLR.org, 2017.

\bibitem[Almer et~al.(2011)Almer, Topham, and Franke]{noc}
O.~Almer, N.~Topham, and B.~Franke.
\newblock A learning-based approach to the automated design of mpsoc networks.
\newblock 2011.

\bibitem[Audibert and Bubeck(2010)]{audibert_best_2010}
J.-Y. Audibert and S.~Bubeck.
\newblock Best arm identification in multi-armed bandits.
\newblock In \emph{{COLT} - 23th Conference on Learning Theory}, 2010.

\bibitem[Auer et~al.(2016)Auer, Chiang, Ortner, and Drugan]{auer_pareto_2016}
P.~Auer, C.-K. Chiang, R.~Ortner, and M.-M. Drugan.
\newblock Pareto front identification from stochastic bandit feedback.
\newblock In \emph{Proceedings of the 19th International Conference on
  Artificial Intelligence and Statistics}. {PMLR}, 2016.

\bibitem[Azizi et~al.(2022)Azizi, Kveton, and Ghavamzadeh]{azizi}
M.~Azizi, B.~Kveton, and M.~Ghavamzadeh.
\newblock Fixed-budget best-arm identification in structured bandits.
\newblock In \emph{Proceedings of the Thirty-First International Joint
  Conference on Artificial Intelligence}. International Joint Conferences on
  Artificial Intelligence Organization, 2022.

\bibitem[Bubeck and Munos(2008)]{bubeck_pure_2008}
S.~Bubeck and G.~Munos, R.and~Stoltz.
\newblock Pure exploration for multi-armed bandit problems, 2008.

\bibitem[Camilleri et~al.(2021)Camilleri, Jamieson, and
  Katz-Samuels]{camilleri21_kernel}
R.~Camilleri, K.~Jamieson, and J.~Katz-Samuels.
\newblock High-dimensional experimental design and kernel bandits.
\newblock In \emph{Proceedings of the 38th International Conference on Machine
  Learning}, Proceedings of Machine Learning Research. PMLR, 2021.

\bibitem[Carpentier and Locatelli(2016)]{carpentier_tight_2016}
A.~Carpentier and A.~Locatelli.
\newblock Tight (lower) bounds for the fixed budget best arm identification
  bandit problem.
\newblock In \emph{29th Annual Conference on Learning Theory}, Proceedings of
  Machine Learning Research. PMLR, 2016.

\bibitem[Chen et~al.(2014)Chen, Lin, King, Lyu, and Chen]{combBandit}
S.~Chen, T.~Lin, I.~King, M.~R. Lyu, and W.~Chen.
\newblock Combinatorial pure exploration of multi-armed bandits.
\newblock In \emph{Advances in Neural Information Processing Systems}. Curran
  Associates, Inc., 2014.

\bibitem[Degenne et~al.(2020)Degenne, M{\'{e}}nard, Shang, and
  Valko]{degenne20gamification}
R.~Degenne, P.~M{\'{e}}nard, X.~Shang, and M.~Valko.
\newblock Gamification of pure exploration for linear bandits.
\newblock In \emph{{International Conference on Machine Learning (ICML)}},
  2020.

\bibitem[Drugan and Nowe(2013)]{drugan_designing_2013}
M.-M. Drugan and A.~Nowe.
\newblock Designing multi-objective multi-armed bandits algorithms: A study.
\newblock In \emph{The 2013 International Joint Conference on Neural Networks
  (IJCNN)}, 2013.

\bibitem[Faizal and Nair(2022)]{faizal2022constrained}
F.~Z. Faizal and J.~Nair.
\newblock Constrained pure exploration multi-armed bandits with a fixed budget,
  2022.

\bibitem[Fiez et~al.(2019)Fiez, Jain, Jamieson, and Ratliff]{tanner}
T.~Fiez, L.~Jain, K.~G. Jamieson, and L.~Ratliff.
\newblock Sequential experimental design for transductive linear bandits.
\newblock In \emph{Advances in Neural Information Processing Systems}. Curran
  Associates, Inc., 2019.

\bibitem[Garivier and Kaufmann(2016)]{garivier_optimal_2016}
A.~Garivier and E.~Kaufmann.
\newblock Optimal best arm identification with fixed confidence.
\newblock In \emph{29th Annual Conference on Learning Theory}. PMLR, 2016.

\bibitem[Jamieson et~al.(2014)Jamieson, Malloy, Nowak, and
  Bubeck]{jamieson_lil_2013}
K.~Jamieson, M.~Malloy, R.~Nowak, and S.~Bubeck.
\newblock lil' ucb : An optimal exploration algorithm for multi-armed bandits.
\newblock In \emph{Proceedings of The 27th Conference on Learning Theory}.
  PMLR, 2014.

\bibitem[Jedra and Proutiere(2020)]{jedra}
Y.~Jedra and A.~Proutiere.
\newblock Optimal best-arm identification in linear bandits.
\newblock In \emph{Advances in Neural Information Processing Systems}. Curran
  Associates, Inc., 2020.

\bibitem[Kalyanakrishnan et~al.(2012)Kalyanakrishnan, Tewari, Auer, and
  Stone]{kalyanakrishnan_pac_2012}
S.~Kalyanakrishnan, A.~Tewari, P.~Auer, and P.~Stone.
\newblock {PAC} subset selection in stochastic multi-armed bandits.
\newblock In \emph{Proceedings of the 29th International Coference on
  International Conference on Machine Learning}. Omnipress, 2012.

\bibitem[Karnin et~al.(2013)Karnin, Koren, and Somekh]{karnin_almost_2013}
Z.~Karnin, T.~Koren, and O.~Somekh.
\newblock Almost optimal exploration in multi-armed bandits.
\newblock In \emph{Proceedings of the 30th International Conference on
  International Conference on Machine Learning}. {JMLR}, 2013.

\bibitem[Katz-Samuels and Scott(2018)]{katz-samuels_feasible_2018}
J.~Katz-Samuels and C.~Scott.
\newblock Feasible arm identification.
\newblock In \emph{Proceedings of the 35th International Conference on Machine
  Learning}. {PMLR}, 2018.

\bibitem[Katz-Samuels and Scott(2019)]{katz-samuels_top_2019}
J.~Katz-Samuels and C.~Scott.
\newblock Top feasible arm identification.
\newblock In \emph{Proceedings of the Twenty-Second International Conference on
  Artificial Intelligence and Statistics}. {PMLR}, 2019.

\bibitem[Katz-Samuels et~al.(2020)Katz-Samuels, Jain, karnin, and
  Jamieson]{katz_samuels_empirical_process}
J.~Katz-Samuels, L.~Jain, z.~karnin, and K.~G. Jamieson.
\newblock An empirical process approach to the union bound: Practical
  algorithms for combinatorial and linear bandits.
\newblock In \emph{Advances in Neural Information Processing Systems}. Curran
  Associates, Inc., 2020.

\bibitem[Kaufmann et~al.(2016)Kaufmann, Capp{{\'e}}, and
  Garivier]{kaufmann_complexity_2014}
E.~Kaufmann, O.~Capp{{\'e}}, and A.~Garivier.
\newblock On the complexity of best-arm identification in multi-armed bandit
  models.
\newblock \emph{Journal of Machine Learning Research}, 2016.

\bibitem[Kiefer and Wolfowitz(1960)]{kiefer_wolfowitz_1960}
J.~Kiefer and J.~Wolfowitz.
\newblock The equivalence of two extremum problems.
\newblock \emph{Canadian Journal of Mathematics}, 1960.

\bibitem[Kim et~al.(2023)Kim, Iyengar, and Zeevi]{kim2023pareto}
W.~Kim, G.~Iyengar, and A.~Zeevi.
\newblock Pareto front identification with regret minimization, 2023.

\bibitem[Kone et~al.(2023)Kone, Kaufmann, and Richert]{kone2023adaptive}
C.~Kone, E.~Kaufmann, and L.~Richert.
\newblock Adaptive algorithms for relaxed pareto set identification.
\newblock In \emph{Thirty-seventh Conference on Neural Information Processing
  Systems}, 2023.

\bibitem[Kone et~al.(2024)Kone, Kaufmann, and Richert]{kone2023bandit}
C.~Kone, E.~Kaufmann, and L.~Richert.
\newblock Bandit {P}areto set identification: the fixed budget setting.
\newblock In \emph{Proceedings of The 27th International Conference on
  Artificial Intelligence and Statistics}, Proceedings of Machine Learning
  Research. PMLR, 2024.

\bibitem[Lattimore and Szepesvári(2020)]{lattimore_bandit_2020}
T.~Lattimore and C.~Szepesvári.
\newblock \emph{Bandit Algorithms}.
\newblock Cambridge University Press, 2020.

\bibitem[Locatelli et~al.(2016)Locatelli, Gutzeit, and
  Carpentier]{locatelli_optimal_2016}
A.~Locatelli, M.~Gutzeit, and A.~Carpentier.
\newblock An optimal algorithm for the thresholding bandit problem.
\newblock In \emph{Proceedings of The 33rd International Conference on Machine
  Learning}, Proceedings of Machine Learning Research. PMLR, 2016.

\bibitem[Lu et~al.(2019)Lu, Wang, Hu, and
  Zhang]{lu_multi_objective_generalized}
S.~Lu, G.~Wang, Y.~Hu, and L.~Zhang.
\newblock Multi-objective generalized linear bandits.
\newblock In \emph{Proceedings of the 28th International Joint Conference on
  Artificial Intelligence}. AAAI Press, 2019.

\bibitem[Pukelsheim(2006)]{pukelsheim_opt_design}
F.~Pukelsheim.
\newblock \emph{Optimal Design of Experiments (Classics in Applied Mathematics)
  (Classics in Applied Mathematics, 50)}.
\newblock Society for Industrial and Applied Mathematics, 2006.

\bibitem[Soare et~al.(2014)Soare, Lazaric, and Munos]{soare}
M.~Soare, A.~Lazaric, and R.~Munos.
\newblock Best-arm identification in linear bandits.
\newblock In \emph{Advances in Neural Information Processing Systems}. Curran
  Associates, Inc., 2014.

\bibitem[Tao et~al.(2018)Tao, Blanco, and Zhou]{tao}
C.~Tao, S.~Blanco, and Y.~Zhou.
\newblock Best arm identification in linear bandits with linear dimension
  dependency.
\newblock In \emph{Proceedings of the 35th International Conference on Machine
  Learning}, Proceedings of Machine Learning Research. PMLR, 2018.

\bibitem[Tsanas and Xifara(2012)]{misc_energy_efficiency_242}
A.~Tsanas and A.~Xifara.
\newblock {Energy efficiency}.
\newblock UCI Machine Learning Repository, 2012.

\bibitem[Xu et~al.(2018)Xu, Honda, and Sugiyama]{xu}
L.~Xu, J.~Honda, and M.~Sugiyama.
\newblock A fully adaptive algorithm for pure exploration in linear bandits.
\newblock In \emph{Proceedings of the Twenty-First International Conference on
  Artificial Intelligence and Statistics}. PMLR, 2018.

\bibitem[Yang and Tan(2022)]{yang}
J.~Yang and V.~Tan.
\newblock Minimax optimal fixed-budget best arm identification in linear
  bandits.
\newblock In \emph{Advances in Neural Information Processing Systems}. Curran
  Associates, Inc., 2022.

\bibitem[You et~al.(2023)You, Qin, Wang, and Yang]{you23a}
W.~You, C.~Qin, Z.~Wang, and S.~Yang.
\newblock Information-directed selection for top-two algorithms.
\newblock In \emph{Proceedings of Thirty Sixth Conference on Learning Theory},
  Proceedings of Machine Learning Research. PMLR, 2023.

\bibitem[Zuluaga et~al.(2013)Zuluaga, Sergent, Krause, and
  Püschel]{zuluaga_active_201}
M.~Zuluaga, G.~Sergent, A.~Krause, and M.~Püschel.
\newblock Active learning for multi-objective optimization.
\newblock In \emph{Proceedings of the 30th International Conference on Machine
  Learning}. {PMLR}, 2013.

\bibitem[Zuluaga et~al.(2016)Zuluaga, Krause, and
  P{{{\"u}}}schel]{zuluaga_e-pal_2016}
M.~Zuluaga, A.~Krause, and M.~P{{{\"u}}}schel.
\newblock e-pal: An active learning approach to the multi-objective
  optimization problem.
\newblock \emph{Journal of Machine Learning Research}, 2016.

\end{thebibliography}
\bibliographystyle{abbrvnat}


 \begin{appendix}
 \onecolumn
 
 \section{OUTLINE}
\label{sec:outline}
In section~\ref{sec:proof_propFbFc}, we prove Proposition~\ref{prop:propFbFc}, which is a crucial result to prove the guarantees of \gege{} in both fixed-confidence and fixed-budget settings. 
 Section~\ref{sec:analysis_fb} proves the fixed-budget guarantees of \gegeFb{}, in particular Theorem~\ref{thm:thm-prob-gege}. In section~\ref{sec:analysis_fc} we prove the fixed-confidence guarantees of \gegeFc{} by proving Theorem~\ref{thm:thm-gege-fc-sc}. 
 Section~\ref{sec:key_lemmas} contains some ingredient concentration lemmas that are used in our proofs. In section~\ref{sec:lower_bounds}, we analyze the lower bounds in both fixed-confidence and fixed-budget settings. In section~\ref{sec:compte_round}, we analyze the properties of Algorithm~\ref{alg:est_gaps} by using some results on G-optimal design. 
 Finally, section~\ref{sec:add_experiments} contains additional experimental results and the detailed experimental setup.

\section{NOTATION}
\label{sec:notation}
We introduce some additional notation used in the following sections.

In the subsequent sections, $r$ will always denote a round of \gege{}, which should be clear from the context. We then denote by $A_r$ active arms at round $r$ and by $\widehat\Theta_r$ the empirical estimate of $\Theta$ at round $r$, computed by a call to Algorithm~\ref{alg:est_gaps}. By convention we let $\max_{\emptyset} = -\infty$. 

For any sub-optimal arm $i$, there exists a Pareto-optimal arm $i^\star$ (not necessarily unique) such that $\Delta_i = \m(i,i^\star)$. More generally given a sub-optimal $i$ we denote by $i^\star$ any arm of $\argmax_{j\in \cS^\star} \m(i,j)$. 

At a round $r$, we let 
\begin{equation}
	\cP_r := \left\{ \forall\; s \in \{1,\dots,r\},\; \forall\; i \in A_s, i\in (\cS^\star)^\complement \cap A_s \Rightarrow i^\star \in A_s \right\},
\end{equation}
with $\cP = \cP_\infty$. In particular, for a sub-optimal arm $i$, $i^\star \in A_s$ should be understood as $A_s \cap (\argmax_{j\in \cS^\star} \m(i,j)) \neq \emptyset$. If for some $\tau$, $\cP_\tau$ is true and $A_{\tau+1} = \emptyset$ then we say that $\cP$ holds.  

\section{PROOF OF PROPOSITION~\ref{prop:propFbFc}}
\label{sec:proof_propFbFc}

We first recall the result.

\propFbFc*

In both the fixed-budget and fixed-confidence setting, the proof proceeds by induction on the round $r$. Before presenting the inductive argument separately in each case, we establish in Appendix~\ref{sec:deviation_gaps} an important result that is used in both cases (Lemma~\ref{lem:lem-gap-ineq}): if $\cP_r$ holds at some round $r$ then, the empirical gaps computed at this round are good estimators of the true PSI gaps.  

To establish this first result, we need the following intermediate lemmas, proved in Appendix~\ref{sec:key_lemmas}.

\begin{restatable}{lemma}{lemIneq}
\label{lem:lem-ineq}
	At any round $r$ and for any arms $i,j \in A_r$ it holds that 
	\begin{align*}
	&\lvert \Mh(i,j;r) - \M(i,j)\lvert \leq \|(\widehat\Theta_r - \Theta)^\T (x_i - x_j)\|_\infty \; \text{and} \\
	 &\lvert \mh(i,j;r) - \m(i,j)\lvert \leq \|(\widehat\Theta_r - \Theta)^\T (x_i - x_j)\|_\infty.
	\end{align*}
\end{restatable}

\begin{restatable}{lemma}{lemDom}
 \label{lem:lem_dom}
At any round $r$, for any sub-optimal arm $i\in A_r$, if $i^\star \in A_r$ and $i^\star$ does not empirically dominate $i$ then $\Delta_i^\star < \| (\widehat \Theta_r - \Theta)^\T (x_i - x_{i^\star})\|_\infty$.
\end{restatable}

\subsection{Deviations of the gaps when $\cP_r$ holds} \label{sec:deviation_gaps}
In this part, we control the deviations of the empirical gaps when proposition $\cP_r$ holds. 
\begin{lemma}
\label{lem:lem-delta-star}
Assume that the proposition $\cP_r$ holds at some round $r$. Then for any arm $i\in A_r$ it holds that 
$$ \left\lvert (\widehat{\Delta}_{i,r}^\star)_+ - (\Delta_i^\star)_+ \right\lvert \leq \left\lvert \widehat{\Delta}_{i,r}^\star - \Delta_i^\star \right\lvert \leq \gamma_{i,r} $$
where $\gamma_{i, r} := \max_{j\in A_r} \|(\widehat\Theta_r - \Theta)^\T (x_i - x_j)\|_\infty $. 
\end{lemma}
\begin{proof}
This inequality is a direct consequence of Lemma~\ref{lem:lem-ineq} and the relation $\lvert x_+ - y_+\lvert \leq \lvert x-y\lvert)$  which holds for any $x,y\in \bR$. 
Note that for a Pareto-optimal arm $i$ we trivially have 
$ (\Delta_i^\star)^+ = 0 = (\max_{j\in A_r} \m(i, j))_+$. And for a sub-optimal arm $i\in A_r$, as $i^\star \in A_r$ (from proposition $\cP_r$) we have $\Delta_i^\star = \m(i,i^\star) = \max_{j\in A_r} \m(i,j)$. Thus for any arm $i\in A_r$ we have 
\begin{eqnarray*}
\left\lvert (\widehat{\Delta}_{i,r}^\star)_+ - (\Delta_i^\star)_+ \right\lvert &=& \left\lvert (\max_{j\in A_r } \mh(i, j; r))_+ -  (\max_{j\in A_r} \mh(i, j))_+\right\lvert, \\
&\leq& \left\lvert (\max_{j\in A_r} \mh(i, j; r)) - (\max_{j\in A_r} \mh(i, j))\right\lvert, \\
&\leq& \max_{j\in A_r} \left\lvert \mh(i, j; r) -  \mh(i, j)\right\lvert, \\
&\leq& \max_{j\in A_r}  \left\|(\widehat\Theta_r - \Theta)^\T (x_i - x_j) \right\|_\infty = \gamma_{i,r}, 
\end{eqnarray*} 
where the last inequality follows from Lemma~\ref{lem:lem-ineq}. 
\end{proof}

\begin{restatable}{lemma}{lemGapIneq}
\label{lem:lem-gap-ineq} If the proposition $\cP_r$ holds at some round $r$ then for any  arm $i\in A_r$,
$$ \widehat\Delta_{i,r} - \Delta_i \geq \begin{cases}
	-2\gamma_{r} &\text{ if } i\in \cS^\star,\\
	-\gamma_{i,r} & \text{ else,}
\end{cases} $$
where $\gamma_{i, r} := \max_{j\in A_r} \|(\widehat\Theta_r - \Theta)^\T (x_i - x_j)\|_\infty $ and $\gamma_r := \max_{i\in A_r} \gamma_{i,r}$. 
\end{restatable}

\begin{proof}
We first prove the result a sub-optimal arm $i$. From the proposition $\cP_r$, as $i\in A_r$ we have $i^\star \in A_r$ so $\Delta_i = \max_{j\in A_r} \m(i,j)$ and we recall that
\begin{equation}
\label{eq:eq-oo11}
\widehat\Delta_{i,r} := \max(\widehat{\Delta}_{i, r}^\star, \widehat{\delta}_{i,r}^\star).	
\end{equation}
Note that by reverse triangle we have for any arm $i\in A_r$ (sub-optimal or not)
\begin{eqnarray}
\left\lvert \left(\max_{j\in A_r} \mh(i, j; r)\right) -  \left(\max_{j\in A_r} \mh(i, j)\right)\right\lvert
 	&{\leq}& \max_{j\in A_r} \lvert \mh(i, j; r) - \m(i, j)\lvert, \\
		&{\leq}& \max_{j\in A_r} \left \| (\widehat\Theta_r - \Theta)^T(x_i - x_j) \right\|_\infty = \gamma_{i,r}. \label{eq:xc1}
\end{eqnarray}
where the last inequality follows from Lemma~\ref{lem:lem-ineq}. 
If $i$ a sub-optimal arm ($i\notin \cS^\star$) then as $\Delta_i = \Delta_i^\star$, it follows $$\widehat\Delta_{i,r} - \Delta_i \geq \widehat\Delta_{i, r}^\star - \Delta_i^\star$$ therefore 
\begin{eqnarray*}
	\widehat\Delta_{i,r} - \Delta_i &\geq& -\lvert \widehat\Delta_{i, r}^\star - \Delta_i^\star \lvert \\
	&=& -\lvert (\max_{j\in A_r} \mh(i, j; r)) -  (\max_{j\in A_r} \mh(i, j))\lvert\\
	&\geq& -\gamma_{i,r} \quad (\text{see } \eqref{eq:xc1}). 
\end{eqnarray*}
Now we assume $i$ is a Pareto-optimal arm ($i\in \cS^\star$) so that $$\Delta_i = \delta_i^\star.$$ 

Combining with \myeqref{eq:eq-oo11} yields  $$\widehat\Delta_{i,r} - \Delta_{i,r} \geq \widehat\delta_{i,r}^\star - \delta^\star_{i,r}, $$ where we recall that 
$$\widehat \delta_{i,r}^\star=\min_{j \in A_r \backslash\{i\}} [\Mh(i,j;r) \;\land\; (\Mh(j,i;r)_+ + (\widehat \Delta_{j, r}^\star)_+)]$$ and $$\delta^\star_i := \min_{j\in [K]\backslash\{i\}}[\M(i,j) \land (\M(j,i)_+ + (\Delta_{j}^\star)_+)].$$
As for any $x,y\in \bR$ we have $\lvert x^+ - y^+\lvert \leq \lvert x-y\lvert$, the following holds for any $i,j\in A_r$
\begin{eqnarray}
	\lvert \Mh(j,i;r)^+ - \M(j,i)^+\lvert &\leq& \lvert \Mh(j,i;r) - \Mh(j,i)\lvert\\
	&\leq& \gamma_{j,r} \label{eq:xc2}. 
\end{eqnarray}
From Lemma~\ref{lem:lem-delta-star} we have for any $j\in A_r$ \begin{equation}
	\label{eq:xc3}
	(\widehat{\Delta}_{j,r}^\star)_+ - (\Delta_j^\star)_+ \geq -\gamma_{j,r}.
\end{equation}
Combining \eqref{eq:xc2} and \eqref{eq:xc3} yields for any $j\in A_r$
\begin{eqnarray}
\label{eq:xc4}
\Mh(j,i;r)_+ + (\widehat \Delta_{j,r}^\star)_+ &\geq&  \M(j,i)_+ + (\Delta_{j}^\star)_+ -2\gamma_{j,r},
\end{eqnarray}
which in addition to $\Mh(j,i;r) \geq  \Mh(j,i)- \gamma_{j,r}$
yields 
$$ [\Mh(i,j;r) \land (\Mh(j,i;r)_+ + (\widehat \Delta_{j,r}^\star)_+)] \geq [\M(i,j) \land (\M(j,i)_+ + (\Delta_{j}^\star)_+)] - 2\gamma_{j,r}$$
for any arm $j\in A_r$. Thus taking the $\min$ over $i \in A_r$ yields 
 
\begin{eqnarray*}
	\widehat \delta_{i,r}^\star &=& \min_{j\in A_r\backslash\{i\}} [\M(i,j;r) \land (\M(j,i;r)_+ + (\widehat\Delta_{j, r}^\star)_+)]\\
	&\geq& \min_{j\in A_r\backslash\{i\}} [\M(i,j) \land (\M(j,i)_+ + (\Delta_{j}^\star)_+)] - 2\gamma_r, \\
	&\geq& \min_{j\in [K]\backslash\{i\}} [\M(i,j) \land (\M(j,i)_+ + (\Delta_{j}^\star)_+)]   - 2\gamma_{r},\\
	&=& \delta_i^\star - 2\gamma_{r}
\end{eqnarray*}
which concludes the proof of this lemma. 
\end{proof}

Building on this result, we  show that $\cP_\infty$ holds in the fixed-confidence and fixed-budget settings.

\subsection{Fixed-budget setting}
We recall the definition of the good event for any $\lambda>0$.
$$ \cE_{\text{fb}}^{r,\lambda} = \left\{ \forall\;  i,j \in A_r: \|(\widehat\Theta_r - \Theta)^\T (x_i-x_j)\|_\infty \leq \lambda \Delta_{n_{r+1}+1} \right\} $$ and $\cE_{\text{fb}}^\lambda:= \cap_{r=1}^{\lceil\log_2(h)\rceil} \cE_{\text{fb}}^{r,\lambda}$. 
We prove that proposition $\cP_\infty$ holds on the event $\cE_\text{fb}^\lambda$ for some any $\lambda \in (0, 1/5)$.

\begin{restatable}{lemma}{lemGapFb}
\label{lem:lem-gap-fb}
	The proposition holds $\cP_\infty$ on the event $\cE_{\mathrm{fb}}^{\lambda}$ for any $\lambda \in (0,1/5)$: at any round $r\in \{ 1, \dots, \lceil \log_2\dimfeat\rceil+1\}$ and for any arm $i\in A_r\cap(\cS^\star)^\complement$, $i^\star \in A_r$. 
	\end{restatable}

\begin{proof}
	We prove $\cP_\infty$ by induction on the round $r$. In the sequel we assume $\cE_\text{fb}^\lambda$ holds. We also assume $\cP_r$ is true until some round $r$. 
As $\cE_\text{fb}^\lambda$ holds, we have by application of Lemma~\ref{lem:lem-gap-ineq}: for any arm $i\in A_r$, 
\begin{equation}
\label{eq:xh1}
	\widehat\Delta_{i,r} - \Delta_i \geq \begin{cases}
		-2\lambda\Delta_{n_{r+1} +1} &\text{ if } i\in \cS^\star \\
		-\lambda\Delta_{n_{r+1} +1} &\text{ else.}
	\end{cases}
\end{equation}
We shall prove that if a Pareto-optimal arm $i$ is discarded at the end of round $r$ then there exists no arm sub-optimal $j\in A_{r+1}$ such that $j^\star =i$. 
Since $i$ is removed and $\lvert A_{r+1} \lvert = n_{r+1}$ there exists $k_r \in A_{r+1}\cup \{ i\}$ such that 
\begin{equation}
\label{eq:xh2}
	\Delta_{k_r} \geq \Delta_{n_{r+1}+1}.
\end{equation}
If $i$ is empirically sub-optimal then as it is discarded we have $$\widehat\Delta_{i,r} = \widehat\Delta_{i,r}^\star \geq \widehat\Delta_{k,r}$$
for any arm $k\in A_{r+1}$. So $\widehat\Delta^\star_{i,r} \geq \widehat\Delta_{k_r, r}$ thus using \eqref{eq:xh1} and \eqref{eq:xh2}  it comes that
\begin{eqnarray*}
	 \max_{q\in A_r \backslash\{i\}} \m(i,q) &\geq& \Delta_{n_{r+1} +1} - 3\lambda\Delta_{n_{r+1} +1}\\&=& (1-3\lambda)\Delta_{n_{r+1} +1}
\end{eqnarray*} 
and the latter inequality is not possible for $\lambda<1/3$ as the LHS of the inequality is negative as $i$ is a Pareto-optimal arm.

Next we assume that $i$ is empirically optimal.
We claim that $j$ is not dominated by $i$. To see this, first note that as $j\in A_{r+1}$ we have 
\begin{equation}
	\label{eq:xh3}
	\widehat\Delta_{i,r}  \geq \widehat\Delta_{j,r}
\end{equation}
so that as $i$ is empirically optimal, if $j$ was empirically dominated by $i$ we would have 
\begin{equation}
\label{eq:xh4}
\widehat\Delta_{i,r} \leq \Mh(j,i;r)_+ +(\widehat\Delta_{j,r}^\star)_+ = \widehat\Delta_{j,r}.	
\end{equation}
Combining \eqref{eq:xh3} and \eqref{eq:xh4} yield $\widehat\Delta_{i,r} = \widehat\Delta_{j,r} $, $i$ is empirically optimal and $j$ is empirically sub-optimal. However our breaking rule ensures that this case cannot occur. Therefore $j$ is not dominated by $i$. But, by assumption, $j$ is such that $j^\star =i$ and we have proved that $i$ does not empirically dominate $j$ so by Lemma~\ref{lem:lem_dom}
$$ \Delta_j \leq \|(\widehat\Theta_r - \Theta)^\T(x_j - x_i)\|_\infty$$
which on the event $\cE_\text{fb}$ yields
\begin{equation}
\label{eq:xh5}
	\Delta_j \leq \lambda\Delta_{n_{r+1} +1}.
\end{equation} 
On the other side, as $i$ is discarded as an empirically optimal arm we have 
$$\widehat\Delta_{i,r} = \widehat\delta_{i,r}^\star \geq \widehat\Delta_{k,r}$$
for any arm $k\in A_{r+1}$. Since $k_r\in A_{r+1}\cup\{i\}$ it comes $\widehat\delta^\star_{i,r} \geq \widehat\Delta_{k_r, r}$ thus using \eqref{eq:xh1} and \eqref{eq:xh2}  yields 

$$ \M(j,i)_+ + \Delta_j \geq \Delta_{n_{r+1}+1} -4\lambda\Delta_{n_{r+1}+1}$$
which further combined with \eqref{eq:xh5} yields 
$$ \M(j,i)_+ \geq (1-5\lambda)\Delta_{n_{r+1}+1}.$$
However, as $j^\star =i$ we have $\M(j,i)_+=0$ so the latter inequality is not possible as long as $\lambda<1/5$. 
Put together, we have proved proved that if $\cP_r$ holds then for any Pareto-optimal arm $i$ which is removed at the end of round $r$, there does not exist an arm $j\in A_{r+1}$ such that $j^\star =i$. So $\cP_{r+1}$ holds. Finally noting that $\cP_r$ trivially holds for $r=1$ we conclude that $\cP_\infty$ holds on the event $\cE_\text{fb}^\lambda$ for any $\lambda<1/5$. 
 \end{proof}
Combining this result with Lemma~\ref{lem:lem-gap-ineq} and assuming $\cE_{\text{fb}}^\lambda$ holds then yields at any round $r\in \{ 1, \dots, \lceil \log_2\dimfeat\rceil\}$ and for any arm $i\in A_r$: 
\begin{equation}
	\widehat{\Delta}_{i,r} - \Delta_i \geq  \begin{cases}
		-2\lambda \Delta_{n_{r+1}+1} & \text{ if } i\in \cS^\star \\
		-\lambda \Delta_{n_{r+1}+1} & \text{ else}, 
	\end{cases} 
\end{equation}
which proves Proposition~\ref{prop:propFbFc} in the fixed-budget setting.

\subsection{Fixed-confidence setting}
We recall below the good events we study in the fixed-confidence setting: 
$$\cE_{\text{fc}}^r = \left\{ \forall\; i,j \in A_r: \|(\widehat\Theta_r - \Theta)^\T (x_i-x_j)\|_\infty \leq \veps_r/2 \right\} $$
 and $\cE_{\text{fc}}:= \cap_{r=1}^\infty \cE_{\text{fc}}^r$. 
\begin{restatable}{lemma}{lemGap}
\label{lem:lem-gap}
	The proposition $\cP_\infty$ holds on the event $\cE_\text{fc}$: at any round $r$ for any arm $i\in A_r\cap(\cS^\star)^\complement$, $i^\star \in A_r$. \end{restatable}
\begin{proof}[Proof of Lemma~\ref{lem:lem-gap}]
We prove the proposition by induction on the round $r$. Note that the proposition $\cP_r$ trivially holds for $r=1$. 

Assume the property holds until the beginning of some round $r$. Let $i\in \cS^\star$ be an optimal arm and assume $i$ is discarded at the end of round $r$. We will prove that there exists no sub-optimal arm $j\in A_{r+1}$ such that $j^\star=i$. Recall that when $i$ is discarded, we have either $i\in S_r$ (empirically optimal) or $i\notin S_r$ (empirically sub-optimal). We analyze both cases below. If $i\notin S_r$ then it holds that 

	$$ \widehat\Delta_{i,r} \geq \veps_r/2,$$ then, as $i\notin S_r$ it follows that $\widehat\Delta_{i,r} = \widehat\Delta^\star_i := \max_{j\in A_r \backslash\{i\}}\mh(i,j;r)$, so
	$$ \max_{j\in A_r \backslash\{i\}}\mh(i,j;r) \geq \veps_r/2$$ which using Lemma~\ref{lem:lem-ineq} and assuming event $\cE_{\text{fc}}^r$ holds would yield $$\max_{j\in A_r\backslash\{i\}}\m(i,j)>0.$$ The latter inequality is not possible as $i\in \cS^\star$ is a Pareto-optimal arm. 
	Therefore, on  $\cE_{\text{fc}}^r$, when $i\in \cS^\star$ is discarded we have $i\in S_r$. 
	
Next, we analyze the case  $i\in S_r$: that is $i$ is discarded and classified as optimal.
In this case it follows from the definition of $\widehat\Delta_{i,r}$ that 
\begin{equation}
\label{eq:msw}
	\min_{j\in A_r\backslash\{i\}} [ \Mh(j, i; r)_+ +(\widehat\Delta_{j,r}^\star)_+] \geq \veps_r.
\end{equation} 
	 Let  $j\in A_{r+1}\cap (\cS^\star)^c$ be such that $j^\star = i$. If $j$ is empirically optimal then $(\widehat\Delta_{j,r}^\star)_+ = 0$ thus 
	$\Mh(j, i; r)_+ \geq \veps_r$. On the contrary, if $j$ is empirically sub-optimal then because it has not been removed at the end of round $r$ it holds that 
	$$\widehat\Delta_{j,r}^\star < \veps_r / 2, $$ which combined with \eqref{eq:msw} yields $\Mh(j, i; r)_+ > \veps_r/2$. Thus, in both cases we have $\Mh(j, i; r)_+ > \veps_r/2$ which using Lemma~\ref{lem:lem-ineq} and assuming event $\cE_{\text{fc}}^r$
	would imply that 
	$$\M(j, i)_+ >0,$$ which is impossible as, by assumption $j^\star=i$, so $j$ is dominated by $i$. 
	
	Put together with what precedes, on $\cE_\text{fc}$, if $\cP_r$ holds then $\cP_{r+1}$ holds.
	Since the property trivially holds for $r=1$ we have proved that the property $\cP_r$ holds at any round when $\cE_\text{fc}$ holds.
\end{proof}

Combining this result with Lemma~\ref{lem:lem-gap-ineq} proves that, on the event $\cE_{\text{fc}}$, for any round $r$ and for any arm $i\in A_r$   
\begin{equation}
\widehat{\Delta}_{i,r} - \Delta_i \geq  \begin{cases}
		-\veps_r & \text{ if } i\in \cS^\star \\
		-\veps_r/2 & \text{ else},
	\end{cases} 	
\end{equation}
which proves Proposition~\ref{prop:propFbFc} in the fixed-confidence setting.

\section{UPPER BOUND ON THE PROBABILITY OF ERROR}
\label{sec:analysis_fb}
In this section, we prove the theoretical guarantees of \gege{} in the fixed-budget setting. We prove Theorem~\ref{thm:thm-prob-gege} and some ingredient lemmas. 
\thmProbGege* 
\begin{proof}[Proof of Theorem~\ref{thm:thm-prob-gege}]
We first prove the correctness of \gegeFb{} on the event $\cE_\text{fb}^\lambda$ for some $\lambda$ small enough.
Let us assume $\cE_\text{fb}^\lambda$ holds which by Proposition~\ref{prop:propFbFc} implies that $\cP_\infty$ holds and at round $r$, we have for any arm $i\in A_r$
\begin{equation}
\label{eq:xo1}
	\widehat\Delta_{i,r} - \Delta_i \geq \begin{cases}
		-2\lambda\Delta_{n_{r+1} +1} &\text{ if } i\in \cS^\star \\
		-\lambda\Delta_{n_{r+1} +1} &\text{ else.}
	\end{cases}
\end{equation} 
We recall the definition of the good event for any $\lambda>0$, 
$$ \cE_{\text{fb}}^{r,\lambda} = \left\{ \forall\;  i,j \in A_r: \|(\widehat\Theta_r - \Theta)^\T (x_i-x_j)\|_\infty \leq \lambda \Delta_{n_{r+1}+1} \right\} $$ and $\cE_{\text{fb}}:= \cap_{r=1}^{\lceil\log_2(h)\rceil} \cE_{\text{fb}}^{r,\lambda}$. 
Applying Lemma~\ref{lem:lem-ineq} on this event then yields for all arms $i,j\in A_r$, 
\begin{align}
	&\lvert \Mh(i,j;r) - \M(i,j)\lvert \leq \lambda \Delta_{n_{r+1}+1} \; \text{and} \label{eq:zqw1}\\
	 &\lvert \mh(i,j;r) - \m(i,j)\lvert \leq \lambda \Delta_{n_{r+1}+1}. \label{eq:zqw2}
	\end{align}
	
Let $i$ be an arm discarded at the end of round $r$. Since $i$ is discarded and $\lvert A_{r+1} \lvert = n_{r+1}$ there exists $k_r \in A_{r+1}\cup \{ i\}$ such that 
\begin{equation}
\label{eq:xo2}
	\Delta_{k_r} \geq \Delta_{n_{r+1}+1}.
\end{equation}
If $i\notin S_r$ that is $i$ is empirically sub-optimal then 
$$\widehat\Delta_{i,r}=\widehat\Delta_{i,r}^\star \geq \widehat\Delta_{k_r, r},$$
then, recalling that 
$$ \widehat\Delta_{i,r}^\star := \max_{j\in A_r \backslash\{ i\}} \m(i,j; r) $$
and further applying \eqref{eq:xo1} to $k_r$ and using \eqref{eq:zqw2} yields 
$$ \max_{j\in A_r \backslash\{ i\}} \m(i,j) \geq (1-3\lambda)\Delta_{n_{r+1}+1}$$  
which for $\lambda<1/3$ implies that $\max_{j\in A_r} \m(i,j)>0$, that is there exists $j\in A_r$ such that $\vmu_i \prec \vmu_j$ so $i$ is a sub-optimal arm. 

Next, assume $i\in S_r$ (i.e, $i$ is empirically Pareto-optimal). 
In this case we have $\widehat\Delta_{i,r} = \widehat\delta^\star_{i,r} \geq \widehat\Delta_{k_r, r}$. We recall that 

$$\widehat\delta^\star_{i,r} = \min_{j \in A_r \backslash\{i\}} [\Mh(i,j;r) \land (\Mh(j,i;r)_+ + (\widehat \Delta_{i, r}^\star)_+) ].$$ 
 Applying  \eqref{eq:xo1} to $k_r$ and using \eqref{eq:zqw1}, it follows that 
$$ \min_{j\in A_r\backslash\{i\}} \M(i,j) \geq  (1-3\lambda)\Delta_{n_{r+1} +1}.$$ 
Thus, for $\lambda<1/3$, we have $\min_{j\in A_r\backslash\{i\}} \M(i,j)>0$. Therefore, no active arm at round $r$ dominates $i$ (based on their true means), which, together with proposition $\cP_\infty$, yields that $i$ is a Pareto-optimal arm (otherwise, we would have $i^\star \in A_r$ that dominates $i$). 

All put together, we have proved that for any $\lambda<1/5$ (we need $\lambda<1/5$ for $\cP_\infty$ to hold), Algorithm~\ref{alg:gegefb} does not make any error on the event $\cE_\text{fb}^\lambda$. It then follows that the probability of error of \gegeFb{} is at most 
\begin{equation}
\label{eq:eq-ff1}
	\inf_{\lambda\in (0, 1/5)} \bP\left((\cE_{\text{fb}}^\lambda)^c\right)
\end{equation}
Now we upper-bound \myeqref{eq:eq-ff1}, which will conclude the proof. 
Let $\lambda\in (0, 1/5)$ be fixed. We have by union bound 
\begin{eqnarray*}
	\bP\left((\cE_{\text{fb}}^\lambda)^c\right) &\leq& \sum_{r=1}^{\lceil \log_2\dimfeat\rceil } \bE\left[\bP\left((\cE_{\text{fb}}^{r,\lambda})^c \lvert A_r\right)\right]\\
	&\leq& \sum_{r=1}^{\lceil\log_2\dimfeat\rceil}  \bE\left[ \sum_{i\in A_r} \bP(\|(\widehat\Theta_r - \Theta)^\T x_i\|_\infty > \frac12\lambda\Delta_{n_{r+1}+1}\lvert A_r) \right] 
\end{eqnarray*}
Note that for $i$ fixed, we can use
Lemma~\ref{lem:lem-concentr-design} with $\kappa=1/3$ and the conditions of this theorem are satisfied as the budget per phase is $T/\log_2(\dimfeat) \geq 45\dimfeat$ (recall from the theorem that \gege{} is run with $T\geq 45\dimfeat\log_2(\dimfeat)$). Thus, applying this theorem yields
\begin{eqnarray*}
	\bP\left((\cE_{\text{fb}}^\lambda)^\complement\right)
	&\leq& 2\dimvec \sum_{r=1}^{\lceil \log_2\dimfeat\rceil} n_r \bE\left[\exp\left( -\frac{\lambda^2\Delta_{n_{r+1}+1}^2 T}{24\sigma^2h_r\log_2\dimfeat\rceil}\right) \right]\\
	&\leq& 2\dimvec \sum_{r=1}^{\lceil \log_2\dimfeat\rceil} n_r \exp\left(-\frac{\lambda^2T\Delta_{n_{r+1}+1}^2}{24\sigma^2 \min(\dimfeat, n_r)\lceil \log_2\dimfeat\rceil}\right), \quad \text{as $h_r \leq \min(n_r, h)$}. 
\end{eqnarray*} 
Then, note that 
\begin{eqnarray*}
\frac{\Delta_{n_{r+1}+1}^2}{\min(\dimfeat, n_r)} &=&
\frac{\Delta_{\lceil{\dimfeat}/{2^r}\rceil+1}^2}{\lceil\dimfeat/2^{r-1}\rceil} \\
&=&\frac{\Delta_{\lceil{\dimfeat}/{2^r}\rceil+1}^2}{\lceil\dimfeat/2^{r}\rceil+1} \frac{\lceil\dimfeat/2^{r}\rceil+1}{\lceil\dimfeat/2^{r-1}\rceil} \\
&\geq& \frac{\Delta_{\lceil{\dimfeat}/{2^r}\rceil+1}^2}{\lceil\dimfeat/2^{r}\rceil+1} \frac{\dimfeat/2^r + 1}{\dimfeat/2^{r-1} +1}\\
&\geq& \frac{\Delta_{\lceil{\dimfeat}/{2^r}\rceil+1}^2}{\lceil\dimfeat/2^{r}\rceil+1} \frac12, 
\end{eqnarray*}
which follows as $(x+1)/(2x+1) = (1/2) + (1/2)/(2x+1)\geq  \frac12$ for $x\geq 0$. 
Therefore, 
\begin{eqnarray*}
\frac{\Delta_{n_{r+1}+1}^2}{\min(\dimfeat, n_r)} &\geq& \frac12  \frac{\Delta_{\lceil{\dimfeat}/{2^r}\rceil+1}^2}{\lceil\dimfeat/2^{r}\rceil+1}\\	
&\geq& \frac{1}{2H_{2, \text{lin}}}. 
\end{eqnarray*}
Thus, 
\begin{eqnarray*}
\bP\left((\cE_{\text{fb}}^\lambda)^c\right) &\leq&	2\exp\left(-\frac{\lambda^2T}{48\sigma^2H_{2, \text{lin}}\lceil \log_2\dimfeat\rceil}+\log(\dimvec) \right) \sum_{r=1}^{\lceil \log_2\dimfeat\rceil} n_r  \\
&\leq& 2\left({K} + {\dimfeat} + {\lceil \log_2\dimfeat\rceil }\right) \exp\left(-\frac{\lambda^2T}{48\sigma^2H_{2, \text{lin}}\lceil \log_2\dimfeat\rceil}+\log(\dimvec) \right) 
\end{eqnarray*}
Finally, it follows that 
$$\inf_{\lambda\in (0, 1/5)} \bP\left((\cE_{\text{fb}}^\lambda)^c\right) \leq 2\left({K} + {\dimfeat} + {\lceil \log_2\dimfeat\rceil }\right) \exp\left(-\frac{T}{1200\sigma^2H_{2, \text{lin}}\lceil \log_2\dimfeat\rceil}+\log(\dimvec) \right),$$
which concludes the proof. 
\end{proof}

\section{UPPER BOUND ON THE SAMPLE COMPLEXITY}
\label{sec:analysis_fc}
We prove the theoretical guarantees in the fixed-confidence setting. We prove the correctness of Algorithm~\ref{alg:gegeFc} and we prove the sample complexity bound of Theorem~\ref{thm:thm-gege-fc-sc} and some key lemmas. 
 We first prove the correctness of the fixed-confidence variant of \gege{}. 
 
 \subsection{Proof of the correctness}

We need to prove that the final recommendation of Algorithm~\ref{alg:gegeFc} is correct: that is we should show that : at any round $r$, $B_r \subset \cS^\star$ and $D_r\subset (\cS^\star)^\complement$.
 
\begin{restatable}{lemma}{lemCorrectFc}
\label{lem:lem-correctness-fc}	
On the event  $\cE_\text{fc}$, Algorithm~\ref{alg:gegeFc} identifies the correct Pareto set.
\end{restatable}
\begin{proof}[Proof of Lemma~\ref{lem:lem-correctness-fc}]
In this part let $\tau$ denotes the stopping time of Algorithm~\ref{alg:gegeFc}. We assume $\cE_\text{fc}$ holds. 

Using Proposition~\ref{prop:propFbFc} : for any round $r\leq \tau$ for any (Pareto) sub-optimal $i\in A_r$ we have $i^\star \in A_r$. We then prove the correctness of the algorithm as follows.

Let $i$ be an arm that is removed at the end of some round $r$. Assume $i \in S_r$ then, as $i$ is discarded and empirically optimal we have $\widehat\Delta_{i,r} = \widehat\delta_i^\star \geq \veps_r.$ In particular, it holds that 
	$$ \min_{j\in A_r\backslash\{ i\}} \Mh(i,j;r) \geq \veps_r $$  which using 
	Lemma~\ref{lem:lem-ineq} on the event $\cE_\text{fc}$ yields $$\min_{j\in A_r\backslash\{ i\} }\M(i,j)>\eps_r/2>0,$$ that is no active arm dominates $i$. Put together with proposition $\cP_\infty$ (cf Lemma~\ref{lem:lem-gap}) the latter inequality yields $i\in \cS^\star$. Now assume we have {$i\notin S_r$}: $i$ is discarded and it is empirically sub-optimal. Then 
$$ \widehat\Delta_{i,r} = \max_{j\in A_r} \mh(i,j;r)\geq \veps_r/2, $$ so using Lemma~\ref{lem:lem-ineq} again on event $\cE_\text{fc}$ it follows that there exists $j\in A_r$ such that $\m(i,j)>0$: that is
$i\notin \cS^\star$.  Put together, we have proved that if $\cE_\text{fc}$ holds then for any arm $i$ discarded at some round $r$, $$i\in B_{r+1} \equi i \in \cS^\star.$$
Note that if $A_\tau$ is non-empty, then it contains a single arm and because $\cP_\infty$ holds, this arm is also Pareto optimal. \end{proof}

Thus, Algorithm~\ref{alg:gegeFc} is correct on $\cE_\text{fc}$. Before proving Theorem~\ref{thm:thm-gege-fc-sc} we need 
Lemma~\ref{lem:lem-card-active} to control the size of the active set $A_r$ in the fixed-confidence setting. 
 \subsection{Controlling the size of the active set}
 \label{sec:proof_lem_card}
  We prove the following result that controls the size of the active set. 
 \lemCardActive*
 \begin{proof}[Proof of Lemma~\ref{lem:lem-card-active}]
 	By Lemma~\ref{lem:lem-gap} we on the event $\cE_{\text{fc}}$: for any round $r$ and for any arm $i\in A_r$, $$ \widehat{\Delta}_{i,r} - \Delta_i \geq  \begin{cases}
		-\veps_r & \text{ if } i\in \cS^\star \\
		-\veps_r/2 & \text{ else.}
	\end{cases}$$ 
Then let $p\in [K]$ and let assume an arm $i\in \{p, \dots, K\}$ is still active at round $r= \lceil \log_2(1/\Delta_p)\rceil$. We have 
$\widehat{\Delta}_{i, r} \geq \Delta_{i} - \eps_r$ with $\eps_r = 1/2^{r+1}$ and $\Delta_i \geq \Delta_p$ which combined with  
$\widehat{\Delta}_{i,r} \geq \Delta_i - \eps_r $ yields 
\begin{equation}
\label{eq:eq-ww1}
 \widehat{\Delta}_{i,r} \geq \Delta_p - \veps_r. 
\end{equation}
As $r=\lceil \log_2(1/\Delta_p)\rceil$, it holds that $2\veps_r \leq \Delta_p$ so 
\myeqref{eq:eq-ww1} yields $\widehat{\Delta}_{i,r} \geq \veps_r$ thus $i$ will be discarded at the end of round $r$ that is any arm $i \in \{p, \dots, K\}$ will be discarded at the end of round $\lceil \log_2(1/\Delta_p)\rceil$.  
 \end{proof}
 
We now prove the main lemma on the sample complexity of \gege{} in the fixed-confidence setting. 
\subsection{Proof of Theorem~\ref{thm:thm-gege-fc-sc}}
\label{sec:full_proof}
We provide an upper bound on the sample complexity of the algorithm.  
\thmGegeFcSc*
 \begin{proof}
 We assume $\cE_\text{fc}$ holds. The correctness of Algorithm~\ref{alg:gegeFc} is then proven in Lemma~\ref{lem:lem-correctness-fc} and 
 Lemma~\ref{lem:lem-card-active} upper-bounds the number of rounds before termination. It remains to bound the sample complexity of the algorithm on $\cE_\text{fc}$ and compute $\bP(\cE_\text{fc})$ to conclude. 
 
By Lemma~\ref{lem:lem-card-active} an upper-bound on $\lvert A_r\lvert$ for some specific rounds. Interestingly we can bound the sample complexity between consecutive "checkpoints rounds''. 
  In what follows, we rewrite the complexity as a sum of number of pulls between these intermediate "checkpoints rounds''.  
Let us introduce the sequence $\{\alpha_s: s\geq 0\}$ defined 
as $\alpha_0=0$ and for any $s\geq 1$, $\alpha_s = \lceil \log_2(1/\Delta_{\lfloor \dimfeat/2^s\rfloor})\rceil$.
We assume {\it w.l.o.g} that the sequence is non-decreasing and that the gaps are bounded in $(0, 1)$ (otherwise, we could start the sequence $(\alpha)_s$ from arms with gap smaller than $1$). Simple calculation shows that $\alpha_{\lfloor\log_2(\dimfeat)\rfloor} = \lceil \log_2(1/\Delta_{1})\rceil$  and 
\begin{equation}
\label{eq:xcvq}
\{1,\dots, \lceil \log_2(1/\Delta_{1})\rceil\} = \bigcup_{s=1}^{\lfloor \log_2(\dimfeat)\rfloor} \llbracket 1+\alpha_{s-1}, \alpha_s \rrbracket.	
\end{equation}
Introducing $$ T_r  := \frac{32(1+3\veps_r)\sigma^2 h_r}{\eps_r^2} \log\left(\frac{\dimvec n_r}{\delta_r}\right),  
$$ where $n_r =\lvert A_r\rvert$, we have $t_r = \lceil T_r \rceil$, so $t_r \leq T_r + 1$. Using \eqref{eq:xcvq} then leads to 
\begin{eqnarray*} 
\sum_{r=1}^{\lceil \log_2(1/\Delta_{1})\rceil} T_r &=& \sum_{s=0}^{\lfloor \log_2(\dimfeat)\rfloor -1} \sum_{r=\alpha_{s}+1}^{\alpha_{s+1}} T_r 	\\
&=:&  \sum_{s=0}^{\lfloor \log_2(\dimfeat)\rfloor -1} N_s
\end{eqnarray*} where 
$N_s = \sum_{r=\alpha_{s}+1}^{\alpha_{s+1}} T_r $ is "the number of arms pulls'' between round $(\alpha_{s}+1)$ and $\alpha_{s+1}$. 

Next we bound the term $N_s$ for $s\in\{0, \dots, \lfloor \log_2(\dimfeat)\rfloor -1\}$. We recall that $\dimfeat_r \leq \min(\dimfeat, n_r)$ as, $n_r=\lvert A_r\lvert$ is the number of active arms at round $r$ and $h_r$ is the dimension of the space spanned by the features of the active arms. Using Lemma~\ref{lem:lem-card-active} on $\cE_\text{fc}$, it holds that for $r\geq \alpha_{s}+1$ 
\begin{equation}
	n_r \leq \begin{cases} 
	K & \text{ if } s = 0\\
	
	\lfloor \dimfeat / 2^s \rfloor -1 & \text{ if } s\geq1
	\end{cases}
\end{equation}
Therefore for $s\in\{0, \dots, \lfloor \log_2(\dimfeat)\rfloor -1\}$ and for any $r\geq \alpha_s +1$, we simply have $\min(h, n_r) \leq \lfloor \dimfeat / 2^s \rfloor$, so $h_r \leq \lfloor \dimfeat / 2^s \rfloor$ and even $h_r \leq \lfloor \dimfeat / 2^s \rfloor - 1$ if $s>0$. In particular, it holds that $$h_r \leq 2 \lfloor \dimfeat / 2^{s+1} \rfloor \quad \text{for } r\geq \alpha_s +1.$$ It then follows that 
\begin{eqnarray}
	\widetilde N_s = {N_s/(32(1+3\veps_1)\sigma^2)}{} &=& \sum_{r=\alpha_{s}+1}^{\alpha_{s+1}} T_r/(32(1+3\veps_1)\sigma^2) \\
	&\leq&  2\lfloor \dimfeat / 2^{s+1} \rfloor\log\left(\frac{K\dimvec}{\delta_{\alpha_{s +1}}}\right)\sum_{r=\alpha_{s}+1}^{\alpha_{s+1}} \frac{1}{\eps_r^2} \\
	&=& 8 \lfloor \dimfeat / 2^{s+1} \rfloor \log\left( \frac{K\dimvec}{\delta_{\alpha_{s+1}}}\right)\sum_{r=\alpha_{s}+1}^{\alpha_{s+1}} 4^r\\
	&\leq&  8 \lfloor \dimfeat / 2^{s+1} \rfloor \log\left( \frac{K\dimvec}{\delta_{\alpha_{s}+1}}\right)\sum_{r=1}^{\alpha_{s+1}} 4^r\\
	&=& \frac{32 \lfloor \dimfeat / 2^{s+1} \rfloor}{3} \log\left( \frac{K\dimvec}{\delta_{\alpha_{s+1}}}\right)
\left({4^{\alpha_{s+1}}} - 1\right)\label{eq:eq-ws1}
\end{eqnarray}
then further using that 
$$ \alpha_s \geq \begin{cases}
	\log_2(1/\Delta_{\lfloor \dimfeat/2^s\rfloor}) & \text{ if } s\geq 1 \\ 0 &\text{ if } s=0 
\end{cases} $$
yields 
$$  {4^{\alpha_{s+1}}} \leq \frac{1}{\Delta_{\lfloor \dimfeat/2^{s+1}\rfloor}^2}$$
which combined with \eqref{eq:eq-ws1} yields \begin{equation}
	\widetilde N_s \leq \frac{32\sigma^2\lfloor \dimfeat / 2^{s+1} \rfloor}{3\Delta_{\lfloor \dimfeat/2^{s+1}\rfloor}^2}  \log\left( \frac{K\dimvec}{\delta_{\alpha_{s+1}}}\right).
\end{equation}
We can now bound $N = \sum_s N_s$ in terms of the sub-optimality gaps: 
\begin{eqnarray}
	\widetilde N &=&\sum_{s=0}^{\lfloor \log_2\dimfeat\rfloor -1} \widetilde N_s\\
	&\leq&  \frac{32\sigma^2}{3} 
	\sum_{s=0}^{\lfloor \log_2\dimfeat\rfloor -1} \frac{ \lfloor \dimfeat / 2^{s+1} \rfloor}{\Delta_{\lfloor \dimfeat / 2^{s+1}\rfloor}^2} \log\left( \frac{\pi^2K\dimvec \lceil\log_2(1/\Delta_{\lfloor \dimfeat / 2^{s+1} \rfloor})\rceil^2}{6\delta}\right) \label{eq:eq-ws3}, \\
	&=& \frac{32\sigma^2}{3} \sum_{s=1}^{\lfloor \log_2\dimfeat\rfloor} \frac{ \lfloor \dimfeat / 2^{s} \rfloor}{\Delta_{\lfloor \dimfeat / 2^{s}\rfloor}^2} \log\left( \frac{\pi^2K\dimvec \lceil\log_2(1/\Delta_{\lfloor \dimfeat / 2^{s} \rfloor})\rceil^2}{6\delta}\right) \label{eq:eq-ws3}
\end{eqnarray}
Then, recalling that by assumption $\Delta_1\leq,\dots,\leq \Delta_K$, one can observe that the mapping from $[K]$ to $(0,\infty)$,  
$$ u \mapsto \frac{ 1}{\Delta_{u}^2} \log\left(\frac{\pi^2K\dimvec \lceil\log_2(1/\Delta_{u})\rceil^2}{6\delta}\right) $$ is non-increasing and it is easy to check that 
$$ \lfloor \dimfeat/2^s \rfloor - \lceil \lfloor \dimfeat/2^s \rfloor/2 \rceil + 1 \geq \frac12 \lfloor \dimfeat/2^s \rfloor $$
 therefore 
 \begin{equation}
 \label{eq:eq-ws2}
 \frac{ \lfloor \dimfeat / 2^s \rfloor}{\Delta_{\lfloor \dimfeat / 2^s\rfloor}^2} \log\left( \frac{\pi^2K\dimvec \lceil\log_2(1/\Delta_{\lfloor \dimfeat / 2^s \rfloor})\rceil^2}{6\delta}\right) \leq 2\sum_{u=\lceil \lfloor \dimfeat/2^s \rfloor/2 \rceil}^{\lfloor \dimfeat/2^s \rfloor}\frac{ 1}{\Delta_{u}^2} \log\left(\frac{\pi^2K\dimvec \lceil\log_2(1/\Delta_{u})\rceil^2}{6\delta}\right) 
 \end{equation}
 Combining \eqref{eq:eq-ws3} and \eqref{eq:eq-ws2} yields 
 \begin{eqnarray} 
 \label{eq:eq-ws4}   
 	N\leq \frac{64\sigma^2}{3} 	\sum_{s=1}^{\lfloor \log_2\dimfeat \rfloor} \sum_{u=\lceil \lfloor \dimfeat/2^s \rfloor/2 \rfloor}^{\lfloor \dimfeat/2^s \rfloor}\frac{1}{\Delta_{u}^2} \log\left(\frac{\pi^2K\dimvec\lceil\log_2(1/\Delta_{u})\rceil^2}{6\delta}\right)
 \end{eqnarray}
 Now let us introduce for any $s$, the set of integers  
 $\cI_s = \llbracket\lceil \lfloor \dimfeat/2^s \rfloor/2 \rceil, \lfloor \dimfeat/2^s \rfloor \rrbracket$. We have $$ \bigcup_{s=1}^{{\lfloor \log_2\dimfeat\rfloor}}\cI_s \subset \{2,\dots, \dimfeat\}.$$   
 We show that for any $p,q\in \{1, \dots, \lfloor \log_2(\dimfeat) \rfloor \}$ if $\lvert p - q \lvert \geq 2$ then $\cI_p \cap \cI_q = \emptyset$. Assuming $p\leq q$ we claim that 
 \begin{equation}
 \label{eq:eq-ws5}
 \lfloor \dimfeat / 2^{p+2} \rfloor < \lceil \lfloor \dimfeat/2^p\rfloor/2\rceil	 \end{equation}
 Assume otherwise, then $
 	\lfloor \dimfeat / 2^{p+2} \rfloor \geq \lceil \lfloor \dimfeat/2^p\rfloor/2\rceil\geq \lfloor \dimfeat/2^p\rfloor/2$ so 
 	$$ \dimfeat /2^{p+1} \geq \lfloor \dimfeat/2^p\rfloor$$
which is impossible since for any $p\in \{0, \dots, \lfloor \log_2(\dimfeat) \rfloor -1 \}$, $\dimfeat /2^p \geq 1$. Therefore we have proved \eqref{eq:eq-ws5} and for any $q\geq p+2$ it holds that 
$$\lfloor \dimfeat / 2^{q} \rfloor  \leq \lfloor \dimfeat / 2^{p+2} \rfloor < \lceil \lfloor \dimfeat/2^p\rfloor/2\rceil$$ thus $\cI_q \cap \cI_p = \emptyset$ and for any $i\in \{2,\dots, \dimfeat\}$, $i$ belongs to no more than $2$ of the subsets $\cI_1,\dots \cI_{\lfloor\log_2h\rfloor}$, thus we have 
\begin{eqnarray}
	\widetilde N&\leq& \frac{64}{3}\sigma^2 	\sum_{s=1}^{\lfloor \log_2\dimfeat \rfloor} \sum_{u=\lceil \lfloor \dimfeat/2^s \rfloor/2 \rfloor}^{\lfloor \dimfeat/2^s \rfloor}\frac{1}{\Delta_{u}^2} \log\left(\frac{\pi^2K\dimvec\lceil\log_2(1/\Delta_{u})\rceil^2}{6\delta}\right)\\
	&\leq& \frac{128}{3}\sigma^2 \sum_{i=2}^\dimfeat \frac{1}{\Delta_{i}^2} \log\left(\frac{\pi^2K\dimvec\lceil\log_2(1/\Delta_{i})\rceil^2}{6\delta}\right)\\
	&\leq& \frac{128}{3} \sigma^2 \sum_{i=2}^\dimfeat \frac{1}{\Delta_{i}^2} \log\left(\frac{\pi^2K\dimvec\log_2(2/\Delta_{i})^2}{6\delta}\right)\\
	&\leq& \frac{256}{3} \sigma^2  \sum_{i=2}^\dimfeat \frac{1}{\Delta_{i}^2} \log\left(\frac{K\dimvec}{\delta} \log_2\left(\frac2{\Delta_{i}}\right)\right)
\end{eqnarray}  
  Then, from Lemma~\ref{lem:lem-gap} it holds that with probability at least $1-\delta$ the sample complexity $N^\delta$ of \gegeFc{} is upper-bounded as 
  $$ \log_2(2/\Delta_1) +  O\lp \sum_{i=2}^\dimfeat \frac{ \sigma^2 }{\Delta_{i}^2} \log\left(\frac{K\dimvec}{\delta} \log_2\left(\frac1{\Delta_{i}}\right)\right) \rp,$$
 \end{proof} where $O(\cdot)$ hides universal multiplicative constant.  
 Therefore, we have shown the sample complexity bound and the correctness on $\cE_\text{fc}$. Thus, proving that $\bP(\cE_\text{fc}) \geq 1-\delta$ will conclude the proof. 
 \newcommand{\gefc}{\ensuremath{\cE_\text{fc}}}
 \newcommand{\gefcr}{\ensuremath{\cE_\text{fc}^r}}
 
\subsection{Probability of the good event $\cE_\text{fc}$.}
 At round $r$, 
 \begin{eqnarray*}
 	\bP\left( (\gefcr)^c  \mid  A_r\right) &\leq& \sum_{i \in A_r} \bP\left(\|(\widehat\Theta_r - \Theta)^\T x_i\|_\infty > \veps_r/4\lvert A_r\right)
 \end{eqnarray*}
 Then, note that at round $r$, Algorithm~\ref{alg:gegeFc} calls \optDesign{} with precision $\veps_r/2$ and budget $t_r$ and  by design we have $t_r \geq 20 h_r /\veps_r^2$, so using Lemma~\ref{lem:lem-concentr-design}, it follows 
 
 \begin{eqnarray*}
 	\bP\left( (\gefcr)^c  \mid  A_r\right) &\leq& 2d  \exp\left( - \frac{t_r \veps_r^2}{32(1+3\veps_r)\sigma^2 \dimfeat_r}\right) \\ 
 	&\leq& \delta_r /\lvert A_r \lvert 
 \end{eqnarray*}
 which follows by plugging in the value of $t_r$. 
 Therefore, by union bound over $A_r$ and $r$ it holds that 
 $\bP\left( \gefc \right) \geq 1-\sum_{r\geq }\delta_r \geq 1-\delta$. This conludes the proof of Theorem~\ref{thm:thm-gege-fc-sc}.  
\subsection{\gegeFc{} for $\varepsilon$-PSI}
\label{ssec:varpsi}
Note that sub-optimal arms with small gaps are close to the Pareto Set. Indeed, given $\varepsilon >0$ and a sub-optimal arm $i$ such that $\Delta_i<\varepsilon$. From the definition of the gap (cf Equation~\ref{eq:def-gap}, Section~\ref{sec:setting}), it follows that for any arm $j \neq i$, $m(i,j)<\varepsilon$, which by definition rewrites as $\min_{c \in \{1,\dots, d\}}[\mu_j(c) - \mu_i(c)]< \varepsilon$. Thus, for any arm $j\neq i$, there exists an objective $c_j$ such that $ \mu_i(c_j) + \varepsilon > \mu_j(c_j)$ , that is $\mu_i + (\varepsilon,\dots, \varepsilon)$ is not dominated by any of the arms $\{ \mu_j: j \neq i \}$.

\citet{auer_pareto_2016} proposed the concept of $\varepsilon$-PSI, which allows practitioners to specify a parameter $\varepsilon \geqslant 0$ to define an indifference zone around the Pareto Set. Be given an instance $\mu := (\mu_1, \dots, \mu_K)$ and its Pareto Set $S^\star$, a set $S_{\varepsilon} \subset [K]$ is an $\varepsilon$-Pareto Set if : $S^\star \subset S_{\varepsilon}$, and for any $i \in S_{\varepsilon}$, $\mu_i + (\varepsilon,\dots, \varepsilon)$ is not dominated by any of $\{ \mu_j: j\neq i\}$. Intuitively, such a set contains all the Pareto-optimal arms but may also include some arms that are close to be Pareto-optimal. 

We prove below that with the modification suggested in the main, the recommended set will be an $\varepsilon$-Pareto Set and, with tiny modifications, our proof extends to cover this case and we could prove that Theorem~\ref{thm:thm-gege-fc-sc} holds with each gap $\Delta_i$ replaced with $\Delta_{i, \varepsilon}:= \max(\Delta_i, \varepsilon/2)$; that is the sample complexity is now upper bounded by $$ \log_2(2/\Delta_{1,\veps}) +  O \left( \sum_{i=1}^h \frac{\sigma^2}{\Delta_{i, \varepsilon}^2} \log\lp \frac{Kd}{\delta}\log\lp\frac{1}{\Delta_{i, \varepsilon}}\rp\rp \right).$$

\paragraph{Sketch of proof} At its core, Lemma~\ref{lem:lem-card-active} proves that (with high probability) any arm $i$ such that $\Delta_i \geqslant 2 \varepsilon_r$ will not be active after round $r$ (cf proof in Section~\ref{sec:proof_lem_card}).  Assume the algorithm stops at some round $\tau$ because $\varepsilon_\tau \leqslant \varepsilon/4$. Then, from the previous observation, if an arm $i$ is still alive at round $\tau$ then $\Delta_i < 2 \varepsilon_{\tau -1}$ and since $ \varepsilon_\tau = \varepsilon_{\tau-1} / 2$ (cf $\varepsilon_r = 1/(2\cdot 2^r)$ in Algorithm~\ref{alg:gegeFc}), $\Delta_i < 4\varepsilon_{\tau }$. By assumption on the stopping, $\varepsilon_\tau \leqslant \varepsilon/4$, so any arm still alive at stopping satisfies $\Delta_i < 4 \varepsilon_{\tau } \leqslant \varepsilon$. Coupling this with the proof of Lemma~\ref{lem:lem-correctness-fc} and the discussion above will prove that the recommended set is an $\varepsilon$-Pareto Set. The case where the stopping occurs because $|A_\tau|\leqslant 1$ is already covered by Lemma~\ref{lem:lem-correctness-fc}.

To prove the sample complexity bound, we note that because of this relaxed stopping condition, Lemma~\ref{lem:lem-card-active} simply holds by replacing every gap $\Delta_i$ with $\Delta_{i, \varepsilon}:= \max(\Delta_i, \varepsilon/2)$, then propagating this change into the proof of Theorem~\ref{thm:thm-gege-fc-sc} (Appendix~\ref{sec:full_proof}) yields the claimed result.
\section{CONCENTRATION RESULTS}
\label{sec:key_lemmas}
In this section we prove some concentration inequalities that are essential to the proofs of others results.  

\lemIneq*
\begin{proof}
We have 
	\begin{eqnarray*}
		\lvert \Mh(i,j;r) - \M(i,j) \lvert  &=&\left \lvert \max_{c}\left[\muh_{i, r}(c) - \muh_{j,r}(c) \right] - \max_c \left[\mu_i(c) - \mu_j(c)\right]\right\lvert, \\
		&\overset{(i)}{\leq}& \max_c \left\lvert (\muh_{i, r}(c) - \muh_{j,r}(c)) - (\mu_i(c) - \mu_j(c))
		 \right\lvert, \\
		 &=& \left\| (\vmuh_{i, r} - \vmuh_{j,r}) - (\vmu_i - \vmu_j)\right\|_\infty, \\
		&{=}& \|(\widehat\Theta_r - \Theta)^\T (x_i - x_j)\|_\infty.
	\end{eqnarray*}
where $(i)$ follows from reverse triangle inequality. The second part of the lemma is a direct consequence of the relation $\M(i,j) = -\m(i,j)$ as well as $\Mh(i,j;r)= -\mh(i,j;r)$ that holds for any pair of arms $i,j$.
\end{proof}

\lemDom*
\begin{proof}
Since $i^\star$ does not empirically dominate $i$ it holds that $\Mh(i,i^\star; r)>0$ so 
$\Mh(i,i^\star; r) - \M(i,i^\star) > -\M(i,i^\star)$. Then noting that $$-\M(i,i^\star) = \m(i,i^\star) = \Delta_i$$
yields $\Mh(i,i^\star; r) - \M(i,i^\star) > \Delta_i$. Therefore 
\begin{eqnarray*}
	\Delta_i = \Delta_i^\star &<& \Mh(i,i^\star; r) - \M(i,i^\star) \\
	&\leq& \| (\widehat \Theta_r - \Theta)^\T (x_i - x_{i^\star})\|_\infty,
\end{eqnarray*}
where the last inequality is a consequence of Lemma~\ref{lem:lem-ineq}. 
\end{proof} 

We recall the following lemma from the main paper. 
\lemPseudoInv*
We actually prove a stronger statement that is stated below.
\begin{lemma}
\label{lem:lem-psinv-gene}
	If the noise $\eta_t$ has covariance $\Sigma \in \bR^{d\times d}$ and $a_1,\dots, a_{N}$ are deterministically. Assuming the set of active arms is $x_1, \dots, x_K$ then for any $x \in \text{span}(\{x_1,\dots, x_K\})$, 
  $\emph{Cov}(\widehat{\Theta}_N ^\T x) =   \| x\|_{ V_N^{\dagger}}^2\Sigma.$ 
\end{lemma}

\begin{proof}
In what follows we let $E:= \text{span}(\{x_1, \dots, x_K\})$ be the space spanned the vectors $x_1, \dots x_K$. As the columns of $B$ forms an orthogonal basis of $E$, $P = B(B^\T B)^{-1}B^\T = B B^\T$  is a matrix that project onto $E$. Therefore, for any $x\in E$ 
$$ \Theta^\T x = \Theta^\T BB^\T x = (B^\T \Theta)^\T B^\T x.$$ Thus recalling that $X_N = (x_{a_1}, \dots, x_{a_N})^\T$ it holds that $X_N \Theta = (X_N B)(B^\T \Theta)$. Rewriting the solution of the least squares leads to 
	\begin{eqnarray*}
		\widehat{\Theta}_N &=& B(B^\T V_N B)^{-1}B^\T X_N^\T (X_N\Theta + H_N) \\
		&=& B(B^\T V_N B)^{-1}B^\T X_N^\T (X_N\Theta) + V_N^{\dagger} X_N^\T H_N \\
		&=& B(B^\T V_N B)^{-1}B^\T X_N^\T (X_N B)(B^\T \Theta) + V_N^{\dagger} X_N^\T H_N \\
		&=& B(B^\T V_N B)^{-1}(B^\T V_N  B) (B^\T \Theta) + V_N^{\dagger} X_N^\T H_N\\
		&=& BB^\T \Theta + V_N^{\dagger} X_N^\T H_N
	\end{eqnarray*}

then for any $x\in E$, as $BB^\T x = x$ it follows that 
\begin{eqnarray*}
	\widehat{\Theta}_N^\T x &=&  \Theta^\T BB^\T x +  (V_N^{\dagger} X_N^\T H_N)^\T x\\
	&=& \Theta^\T x + (V_N^{\dagger} X_N^\T H_N)^\T x
\end{eqnarray*}
thus we have for $x\in E$, 
\begin{equation}
	(\widehat{\Theta}_N - \Theta)^\T x = (V_N^{\dagger} X_N^\T H_N)^\T x. 
\end{equation}
Computing the covariance follows as 
\begin{eqnarray}
	\text{Cov}((\widehat{\Theta}_N - \Theta)^\T x) &=& \bE\left[ (V_N^{\dagger} X_N^\T H_N)^\T x x^\T(V_N^{\dagger} X_N^\T H_N) \right] \\
	&=& \bE\left[ H_N^\T  \tilde x \tilde x ^\T H_N\right]
	\label{eq:eq-mm0}
\end{eqnarray}
where $\tilde x := X_N V_N^{\dagger} x $. Letting $h_i^\T$ denotes the $i$-th row of $H_N^\T$, for each $i,j$ 
\begin{eqnarray}
	\bE[h_i^\T \tilde x \tilde x^\T h_j] &=& \tilde x^\T \bE[h_i h_j^\T] x\\
	&=& \tilde x^\T \sigma_{i,j} \tilde x \label{eq:eq-mm1}
\end{eqnarray}
where $\Sigma := (\sigma_{r,s})_{r,s\leq \dimvec}$ and the last line follows since for any $t,t'\leq N$ by independence of successive observations we have 
$ \bE[h_i(t) h_j(t')] = \delta^\text{kro}_{t, t'}\sigma_{i,j}$. 
Combining \myeqref{eq:eq-mm1} with \myeqref{eq:eq-mm0}
 yields 
 $$\text{Cov}((\widehat{\Theta}_N - \Theta)^\T x)  = \Sigma \tilde x^\T \tilde x $$ 
 then further noting that  
 \begin{eqnarray*}
 	\tilde x^\T \tilde x &=& x^\T V_N^{\dagger} X_N^\T X_N V_N^{\dagger} x\\
 	&=& x^\T B(B^\T V_N B)^{-1}B^\T V_N B(B^\T V_N B)^{-1}B^\T x \\
 	&=& x^\T V_N^{\dagger}  x = \|x\|^2_{V_N^{\dagger}}
 \end{eqnarray*} 
 concludes the proof. 
 \end{proof}
 The following results is proven in Appendix~\ref{sec:compte_round}.
 \lemConcentrDesign*
\section{LOWER BOUNDS}
\label{sec:lower_bounds}

Before proving the lower bounds, we illustrate the PSI and the quantities $\M, \m$ on Fig.5

\definecolor{yellow2}{rgb}{0.8862745 , 0.84313726, 0.}
\definecolor{red2}{rgb}{0.9607843 , 0.3137255 , 0.07450981}
\definecolor{green3}{rgb}{0.        , 0.43529412, 0.52156866}
\definecolor{blue}{rgb}{0.08627451211214066, 0.125490203499794, 0.23529411852359772}
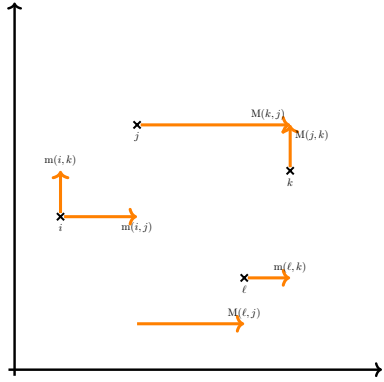
\begin{figure}[ht]
\label{fig:illusGaps}
\centering
\resizebox{0.3\linewidth}{!}{
\begin{tikzpicture}
   \draw [line width=0.8mm, draw=black, ->] (-6.2,-6) -- (6,-6) node[right, black] {};
  \draw [line width=0.8mm, draw=black, ->] (-6,-6.2) -- (-6,6) node[above, black] {};
  
\node[cross out, draw, outer sep =0mm,line width=0.65mm, label=below:{$i$}] (i) at (-4.5, -1) {};
\node[cross out,  draw, outer sep =0mm,line width=0.65mm, label=below:{$j$}] (j) at (-2.0, 2) {};
\node[cross out,  draw, outer sep =0mm,line width=0.65mm, label=below:{$k$}] (k) at (3., 0.5) {};
\node[cross out,  draw, outer sep =0mm,line width=0.65mm, label=below:{$\ell$}] (l) at (1.5, -3){};

\draw[->,line width=1mm, draw = orange](i) -- (-2, -1) node[below] {$\m(i,j)$};
\draw[->,line width=1mm, draw = orange](i) -- (-4.5, 0.5) node[above] {$\m(i,k)$};
\draw[->,line width=1mm, draw = orange](l) -- (3, -3) node[above] {$\m(\ell,k)$};
\draw[->,line width=1mm, draw = orange](k) -- (3, 2) node[anchor=north west] {$\M(j,k)$};
\draw[->,line width=1mm, draw = orange](j) -- (3, 2) node[anchor=south east] {$\M(k,j)$};
\draw[->,line width=1mm, draw = orange](-2.0, -4.5) -- (1.5, -4.5) node[above] {$\M(\ell, j)$};
\end{tikzpicture}}
\captionof{figure}{PSI gaps and distances}
\label{fig:lbd-illusxx}	
\end{figure}

We note that, in this instance $\Delta_i = \m(i, j)$ and by increasing $i$ by $\Delta_i$ on both $x$ and $y$ axes it will become non-dominated. 

We also have $\Delta_\ell = \m(\ell, j)$. As $\ell$ is only dominated by $j$, if is it translated by $\m(\ell, j)$ on the $x$-axis it will become Pareto optimal. 

For Pareto-optimal arms $k,j$, $\delta^+_{k} = \delta^+_{j} = \M(j, k)$. As $k$ dominates both $i$ and $\ell$ its margin to sub-optimal arms is $\delta_k^- = \min(\Delta_i, \Delta_\ell)$ and we have 
$\delta_j^- = \min(\M(\ell, j) + \Delta_\ell, \Delta_i)$. 

Observe that for both $j, k$, $\Delta_j = \Delta_k = \M(j, k)$. If $k$ is translated by $\M(j, k)$ on the $y$-axis it will dominate $j$. Similarly, if $j$ is translated by $-\M(j, k)$ on the $y$-axis, it will be dominated by $k$.

We now prove minimax lower bounds in both fixed-confidence and fixed-budget settings. We recall the lower-bound below for un-structured PSI in the fixed confidence setting. \begin{theorem}[Theorem~17 of \cite{auer_pareto_2016}]
\label{thm:thm-lbd-auer}
	For any set of operating points $\vmu_i \in [1/4, 3/4]^\dimvec$, $i=1,\dots,K$, there exist distributions $(\cD_i)_{1\leq i\leq K}$ such that with probability at least $1-\delta$, any $\delta$-correct algorithm for PSI requires at least 
	$$ \Omega\left( \sum_{i=1}^K \frac{1}{\widetilde{\Delta}_i^2} \log(\delta^{-1})\right)$$
	samples to identify the Pareto set. Where for any sub-optimal arm $\widetilde{\Delta}_i = \Delta_i$ and for an optimal arm $\widetilde{\Delta}_i = \delta_i^+$. 
\end{theorem}
In particular, there exist instances where $\Delta_i = \delta_i^+$ for any Pareto-optimal arm $i$. Thus, this result shows that $H_1$ is a good proxy to measure the complexity of PSI in the fixed-confidence setting.

The proof of such results often rely on the celebrated change of distribution technique (see e.g \cite{kaufmann_complexity_2014}) which given the instance $\nu:= (\nu_1, \dots, \nu_K)$ shifts the mean of $\nu_i$ for an arm $i$ while keeping the others fixed constant. However, in linear PSI, the arms' means are correlated through $\Theta$. So, in general, Theorem~\ref{thm:thm-lbd-auer} does not directly apply to linear PSI. We recall below our lower bound for linear PSI in the fixed-confidence setting. 

\thmLbdFc*
\begin{proof}[Proof of Theorem~\ref{thm:thm-lbd-fc}]
The idea of the proof is to transform an unstructured bandit instance into a linear PSI instance. Let $\nu$ be a bandit instance with $K\geq 2$ arms and dimension $\dimvec\geq 1$ and with means $\vmu_1, \dots, \vmu_K \in [0, 1]^\dimvec$. Let $e_1,\dots e_\dimfeat$ denote the canonical basis of $\bR^\dimfeat$. We define a linear PSI instance $\nu_\text{lin}$ with features 
$$ 	x_i =  \begin{cases}
e_i & \text{ if } i\leq \dimfeat \\ \boldsymbol{0} & \text{ else.}
\end{cases}$$
We assume that the learner knows that $\vmu_i \in [0, 1]^\dimvec$ for any arm $i$. We claim that with this information an "efficient'' algorithm for PSI should not pull arms from $\{\dimfeat+1, \dots, K\}$. To see this, first note that these arms will be sub-optimal so $\cS^\star \subset [\dimfeat]$. Moreover, even if an arm $i \in \{\dimfeat+1,\dots, K\}$ dominates another arm $j\in \{1, \dots, \dimfeat\}$, as $j$ is not Pareto-optimal there exits another arm $j^\star \in \cS^\star \subset \{1,\dots, \dimfeat\}$ which dominates $j$ with a larger margin, so is "cheaper'' to pull. Therefore the complexity of $\nu_\text{lin}$ reduces to the complexity of a linear bandit $\tilde{\nu}_\text{lin}$ with only $\dimfeat$ arms. As the features in $x_1,\dots, x_\dimfeat$ forms the canonical $\bR^\dimfeat$ basis, $\tilde{\nu}_\text{lin}$ reduces to an un-structured bandit instance with (un-correlated) means $\tilde{\vmu}_i = \Theta^\T x_i$, $i=1,\dots, \dimfeat$. Therefore, by choosing $\vmu_1, \dots, \vmu_\dimfeat \in [1/4, 3/4]^\dimvec$, we can apply Theorem~\ref{thm:thm-lbd-auer} to $\tilde{\nu}_\text{lin}$.
\end{proof}
 
The result proven above holds for a class of instances $\cB(K, d, h)$ with the covariates defined as above and with matrix coefficients in $[1/4, 3/4]^d$.  

For the fixed-budget setting, \cite{kone2023bandit} proved a lower bound for a class of instances. We recall their result below after introducing some notation. 

Their lower bound applies to the class of instances $\cB$ defined as follows. 
$\cB$ contains the instances such that each sub-optimal arm $i$ is only dominated by a Pareto-optimal arm denoted by $i^\star$ and that for each optimal arm $j$ there exists a unique sub-optimal arm which is dominated by $j$, denoted by $\underline j$. Moreover, for any instance in $\cB$ the authors require its Pareto-optimal arms not to be close to the sub-optimal arms they don't dominate: for any sub-optimal arm $i$ and Pareto-optimal arm $j$ such that $\vmu_i\nprec \vmu_j$, 
$$ \M(i, j)\geq 3\max(\Delta_i, \Delta_{\underline j}).$$

Let $\nu:= (\nu_1, \dots, \nu_K)$ be an unstructured instance whose means belongs to $\cB$ and with isotropic multi-variate normal arms $\nu_i \sim \cN(\vmu_i, \sigma^2 I)$. For every $i\in [K]$, define the alternative instance $ \nu^{(i)} := (\nu_1, \dots \nu_i^{(i)}, \dots, \nu_K)$ in which {\it only} the mean of arm $i$ is shifted:
\begin{equation}
	\vmu_i^{(i)} := \begin{cases}
		\vmu_i - 2\Delta_i
		\tilde e_{d_{\underline i}} & \text{ if } i \in \cS^\star(\nu) ,\\
		\vmu_i + 2\Delta_i \tilde e_{d_i}&\text{ else},
	\end{cases}
\end{equation}
where $\tilde e_1,\dots, \tilde e_\dimvec$ denotes the canonical basis of $\bR^\dimvec$ and for any arm $i$, $d_i :=\argmin_{c\in [d]} [\mu_{i^\star}(c) - \mu_i(c)]$. Defining $\nu^{(0)}:=\nu$, the theorem below holds.
\begin{theorem}[Theorem~5 of \cite{kone2023bandit}]
\label{thm:thm-lbd-kone}
Let $\nu=(\nu_1, \dots, \nu_K)$ be an instance in $\cB$ with means $\vmu:=(\vmu_1 \dots \vmu_K)^\T$ and $\nu_i \sim \cN(\vmu_i, \sigma^2I)$. For any algorithm $\cA$, there exists $i\in \{0, \dots, K\}$ such that $ H(\nu^{(i)}) \leq  H(\nu)$ and the probability of error $\cA$ on $\nu^{(i)}$ is at least  
$$  \frac14 \exp\left( - \frac{2T}{\sigma^2 H(\nu^{(i)})}\right).$$ 
\end{theorem}
As explained above for the fixed-confidence setting. The proof of this lower bound also uses the change of distribution lemma. In  the instances $\nu^{(i)}$ introduced above, it is crucial that only the mean of arm $i$ has changed w.r.t $\nu^{(0)}$. Therefore, Theorem~\ref{thm:thm-lbd-kone} does not apply to general instances in linear PSI. We recall our lower-bound for linear PSI in the fixed-budget. 
\thmLbdFb*
\begin{proof}[Proof of Theorem~\ref{thm:thm-lbd-fb}]
Let $H$ be fixed and recall that $\mathbb{W}_H : \{ \nu_\text{lin} : H_{2, \text{lin}}(\nu)\leq H \}$ is the set of linear PSI instances with complexity less than $H$. The proof of Theorem~\ref{thm:thm-lbd-fb} follows similar lines to Theorem~\ref{thm:thm-lbd-fc}. Let $\nu$ be an un-structured bandit instance with $K\geq 2$ arms, dimension $\dimvec\geq 1$, with means $\vmu_1, \dots, \vmu_K \in [0, 1]^\dimvec$ and such that $H_2(\nu)\leq H$. We construct a linear PSI instance $\nu_{\text{lin}}$ from an unstructured multi-dimensional instance $\nu$ by setting $x_i := e _i$ for any $i\leq \dimfeat$ and for $i>\dimfeat$, $x_i = \boldsymbol{0}$ where $e_1, \dots , e_\dimfeat$ is the  canonical $\bR^\dimfeat$-basis. 
We also assume that the agent knows that $\vmu_i \in [0, 1]^\dimvec$ for any arm $i$. For $\nu_\text{lin}$ the arms $\{\dimfeat+1, \dots, K\}$ are necessarily sub-optimal so $\cS^\star \subset [\dimfeat]$ thus to identify the Pareto set and efficient algorithm should reduce to pull arms in $\{1,\dots, \dimfeat\}$.  Indeed, as explained in the proof of Theorem~\ref{thm:thm-lbd-fc} even if an arm $i \in \{\dimfeat+1,\dots, K\}$ dominates another arm $j\in \{1, \dots, \dimfeat\}$, as $j$ is not Pareto-optimal there exits another arm $j^\star \in \cS^\star \subset \{1,\dots, \dimfeat\}$ which is "cheaper'' to pull as it dominates $j$ with a larger margin. $\nu_\text{lin}$ reduces to a linear bandit $\tilde{\nu}_\text{lin}$ with only $\dimfeat$ arms and since the features $x_1,\dots, x_\dimfeat$ forms the canonical basis of $\bR^\dimfeat$, $\tilde{\nu}_\text{lin}$ is an un-structured bandit instance with (un-correlated) means $\tilde{\vmu}_i = \Theta^\T x_i$, $i = 1,\dots, \dimfeat$. Therefore, by choosing $\tilde{\nu} := (\nu_1, \dots, \nu_\dimfeat)$ that belongs to $\cB$, we can apply Theorem~\ref{thm:thm-lbd-kone} which yields 
$$ \max_{i\in \{0, \dots, K\}} \bP_{\tilde\nu^{(i)}}(S_T^\cA \neq \cS^\star(\tilde \nu^{(i)}))\geq \frac14\exp\left(-\frac{2T}{H\sigma^2}\right) $$
where by construction $\tilde\nu^{(i)}$ (see construction above) is also a linear PSI instance. Then further  noting that 
$H\geq H_2(\nu) \geq H_2(\tilde \nu)$ and by Theorem~\ref{thm:thm-lbd-kone} for any $i\leq \dimfeat$
$H_{2, \text{lin}}(\tilde \nu) \geq H_2(\tilde\nu^{(i)})$. Then recalling that $\nu_\text{lin}$ is equivalent to $\tilde \nu$ it comes 
$$ \inf_{\cA}\sup_{\nu \in \mathbb{W}_H}\bP_\nu(S_T^\cA \neq \cS^\star(\nu))\geq \frac14\exp\left(-\frac{2T}{H\sigma^2}\right),$$  which is the claimed result. \end{proof}
 \section{COMPUTATION AND ROUNDING OF G-OPTIMAL DESIGN}
\label{sec:compte_round}
In this section, we discuss the results related to the G-design and the rounding. 
In what follows, let $S \subset [K]$ be a set of arms. To ease notation we re-index the arms of $S$ by assuming $S := \{ 1, \dots, \lvert S \rvert \}$. Let $N$ be the allocation budget (the total number of pulls of arms in $S$) and $\kappa \in (0, 1/3]$ the parameter of the rounding algorithm to be fixed. $h_S = \dimspan{S}$ is the dimension of the space spanned by the covariates of $S$. $\cX_S := (x_{i}, i\in S)^\T$ and $B_S:= (u_1, \dots, u_m)$ is the matrix formed with the first $m = h_S = \rank{S}$ columns of $U$, the matrix of left singular vectors of $\cX_S^\T$ obtained by singular value decomposition. We recall that  for $N$ pulls of arms in $[S]$, letting $T_i(N)$ be number of samples collected from arm $i$, 
\begin{equation}
V_N^{\dagger} := B_S (B_S^\T V_NB_S)^{-1}B_S^\T \quad \text{and} \quad V_N := \sum_{i=1}^{K} {T_i(N)} { x_{i}} x_{i}^\T. 
\end{equation}

As from Lemma~\ref{lem:lem-pseudo-inv} the statistical uncertainty from estimating the mean of arm $i$ scales with $\|x_i\|_{V_N^\dagger}$, a call to \optDesign{}($S$, $N$, $\kappa$) is meant to estimate the hidden parameter $\Theta$ by collecting $N$ samples from arms in $S$ according to an approximation of the solution of the following problem (ordinal $G$-optimal design):

\begin{equation}
\label{eq:eq-px1}
\begin{aligned}
\argmin_{s \in [0,\dots,N]^{\lvert S\lvert}}\max_{i\in S} \quad & \| x_i\|_{(V^s)^{\dagger}}\\
\textrm{s.t.} \quad & \sum_{i\in S} s(i) = N\;.
\end{aligned}
\end{equation}

Finding such an optimal design with integer values is an NP-hard problem \citep{zhu}. Instead, its continuous relaxation (obtained by normalizing by $N$), amounts to finding an allocation $\omega$ that minimizes 
\begin{equation}
\label{eq:eq-oo1}
\max_{i\in S} (B_S^\T x_i)^\T \left(\sum_{i\in S} \omega(i) B_S^\T x_i x_i^\T B_S \right)^{-1}B_S^\T x_i, 	
\end{equation}
which reduces to compute a G-optimal allocation on the covariates $B_S ^\T x_i, i \in S:$ 
\begin{equation}
\label{eq:eq-pp1}
w^\star_S \in \argmin_{\omega \in \boldsymbol{\Delta}_{\lvert S\lvert}}\max_{i\in S} \|\widetilde x_i\|_{(\widetilde V^\omega)^{-1}}^2,	 \text{ and } \widetilde V^\omega := \sum_{i\in S} \omega(i) \widetilde x_i \widetilde x_i^\T.
\end{equation}
This is a convex optimization problem on the probability simplex of $\bR^{\lvert S \rvert}$. Efficient solvers have been used in the literature for linear BAI and experiment design optimization see (e.g \cite{tanner, soare}). In this work, we follow \cite{zhu} and we recommend an entropic mirror descent algorithm to solve \myeqref{eq:eq-pp1}, which is recalled as Algorithm~\ref{alg:entropic_md} for the sake of completeness. 

Then, computing an integer allocation whose value is close to the optimal value of \myeqref{eq:eq-pp1} can be done in different ways. \cite{tao} and \cite{camilleri21_kernel} use a stochastic rounding: they sample N arms from $S$ following the distribution $\omega^\star_S$ and use a novel estimator different from the least-squares estimate. \cite{yang, azizi} use floors and ceilings of $N\omega^\star_S$. Although practical, it is known that the value of such rounded allocations can deviate a lot from the optimal value of  \myeqref{eq:eq-px1} \citep{tao}. 

\begin{algorithm}[htb] 
\caption{Entropic mirror descent algorithm for computing $w^\star_S$ \cite{tao}}
   \label{alg:entropic_md}

   {\bfseries Input:} {A set of arms $S$ and covariates $(\widetilde x_i, i\in S)$, tolerance $\veps$ and Lipschitz constant $L_f$}
  
  {\bfseries Initialize:} {$t\gets 1$ and $w^{(1)} \gets (1/\lvert S\lvert, \dots, 1/\lvert S\lvert)$}
  
   \While{$\lvert \max_{ i \in S} \widetilde x_i^{\T} (\widetilde V^{w^{(t)}})^{-1}\widetilde x_i - \dimfeat_S\lvert \geq \veps$}{
   
    {\bfseries set} {$\eta_t \gets \frac{\sqrt{2 \ln N}}{L_f} \frac{1}{\sqrt{t}}$}
    
  {Compute gradient} {$g_i^{(t)} \leftarrow \operatorname{Tr}\left(\widetilde V\left(w^{(t)}\right)^{-1}\left(\widetilde x_i \widetilde x_i^T\right)\right)$}
  
  { Update} {$w_i^{(t+1)} \leftarrow \frac{w_i^{(t)} \exp \left(\eta_t g_i^{(t)}\right)}{\sum_{i=1}^N w_i^{(t)} \exp \left(\eta_t g_i^{(t)}\right)}$}
  
  {$t\gets t+1$}
  }
 
    {\bfseries return:} {$w^{(t)}$}

\end{algorithm}

\cite{zhu} proposed an efficient rounding procedure that guarantees that the value of the returned integer allocation is within a small factor of the optimal value of \myeqref{eq:eq-pp1}. Before recalling their result, we introduce the notation $F_S(s) := \max_{i\in S} \|x_i\|_{(V^{s})^{\dagger}}^2$.  

We recall the celebrated Kiefer-Wolfowitz equivalence theorem below.

\begin{theorem}[Restatement of \cite{kiefer_wolfowitz_1960}]
\label{thm:thm-kiefer-equiv} Let covariates $\{x_i:i\in S \} \subset \bR^\dimfeat$ and for any $\omega \in \boldsymbol{\Delta}_{\lvert S \rvert}$ define $V^\omega = \sum_{i\in S} \omega(i) x_ix_i^\T$ and when $V^\omega$ is non-singular $f(x; \omega):= x^\T (V^\omega)^{-1} x$.  
 The following two extremum problems: 
 \begin{enumerate}[a)]
 	\item $\omega$ maximizing $\text{det}(V^\omega)$ 
 	\item $\omega$ minimizing $\max_{i \in S} f(x_i; \omega)$\label{eq:eq-pw1}
 \end{enumerate}
 are equivalent and a sufficient condition to satisfy \myeqref{eq:eq-pw1} is $\max_{i\in S} f(x_i, \omega) =\dimfeat$, which is satisfied when the covariates $\{x_i:i\in S \} $ span $\bR^h$. 
\end{theorem}

\begin{theorem}[reformulated; rounding method of \cite{zhu}]
\label{thm:thm-rounding}
 Suppose $\kappa \in(0,1 / 3]$ and  $N \geq 5 
  \dimfeat_S / \kappa^2$. Let $\omega^\star_S = \argmin_{\omega \in \boldsymbol{\Delta}_{S}} F_S(\omega)$. Then, there exists an algorithm that outputs an integer allocation $s^\star$ satisfying $$
{s}^\star \in \boldsymbol{\mathcal{D}}_{S, N} \quad \text { and } \quad F_S({s}^\star) \leq(1+6 \kappa) \frac{F_S(\omega^\star_S)}{N} $$
 where $\boldsymbol{\cD}_{S, N} := \{ s\in \{ 0, \dots, N\}^{\lvert S\lvert} : \sum_{i\in S}s(i) = N \}$. This algorithm runs in time complexity $\widetilde{O}\left(N\lvert S\lvert \widetilde\dimfeat^2\right)$. 
\end{theorem}

We refer to a call to this algorithm as \texttt{ROUND}($N, \{\widetilde x_i, i \in S\}, \omega_S^\star, \kappa$). It returns an integer allocation $s^\star = (s^*(1),\dots,s^*(|S|))$ from which we can immediately deduce a list of arms to pull (the first arm in $S$ replicated $s^*(1)$ times, the second replicated $s^*(2)$ times, etc.).  
 
Simple arguments from linear algebra show that the $\dimfeat_S$ columns of $B_S$ form a basis of span($\{x_i: i\in S\}$), hence $\{ B_S^\T x_i: i\in S \}$  spans $\bR^{\dimfeat_S}$. Using Theorem~\ref{thm:thm-kiefer-equiv} applied to the covariates $\{B_S^\T x_i: i\in S\}$ yields $$ F_S(\omega^\star_S) = \dimfeat_S$$ 
and thus the integer allocation $s^\star$ output by \texttt{ROUND}($N, \{\widetilde x_{i},i\in S\}, \omega_S^\star, \kappa$) satisfies for $N\geq 5{\dimfeat_S}/\kappa^2$, $$ F(s^\star) \leq (1+6\kappa) \frac{\dimfeat_S}{N},$$
which is stated below.  
\begin{lemma}
\label{lem:lem-alloc-psinv}
Let $S\subset [K]$, $\kappa \in (0, 1/3]$ and $N\geq 5\dimfeat_S/\kappa^2$ where 
$\dimfeat_S = \text{dim(span($\{x_i: i\in S\}$))}$.  The  
allocation $\{ T_i(N) : i\in S\}$ computed by \optDesign($S$, $N$, $\kappa$) to estimate $\Theta$ satisfies 
 $$ \max_{i\in S} \|x_i\|^2_{V_N^\dagger} \leq (1+6\kappa)\frac{\dimfeat_S}{N}.$$
\end{lemma}
Building on this result, we derive the following concentration result. 
\lemConcentrDesign*
\begin{proof}[Proof of Lemma~\ref{lem:lem-concentr-design}]
We recall that, by assumption, the vector noise has $\sigma$-sub-gaussian marginals. 
	From the proof of Lemma~\ref{lem:lem-psinv-gene} it is easy to see that for any $i\in S$, the marginals of $(\Theta-\widehat\Theta)x_i$ are $\sigma \|X_N^\T V_N^{\dagger}x_i\|_2$-sub-gaussian. Then, direct calculation shows that 
	\begin{eqnarray*}
		\|X_N^\T V_N^{\dagger}x_i\|_2^2 &=& x_i^\T V_N^{\dagger} V_N V_N^{\dagger}x_i\\ \\
		&=& x_i^\T \left(B_S(B_S^\T V_N B_S)^{-1}B_S^\T\right) V_N \left(B_S(B_S^\T V_N B_S)^{-1}B_S^\T\right) x_i \\
		&=& x_i^\T B_S(B_S^\T V_N B_S)^{-1}B_S^\T x_i \\
		&=& x_i^\T V^{\dagger}_N x_i  = \|x_i\|_{V_N^\dagger}^2.
	\end{eqnarray*}
	Therefore, by concentration of sub-gaussian variables (see e.g \cite{lattimore_bandit_2020}) we have for $i$ fixed, 
	\begin{eqnarray*}
	\bP(\|(\Theta - \widehat\Theta)^\T x_i \|_\infty \geq \veps) &\leq& 2\dimvec \exp\left( - \frac{\eps^2}{2\sigma^2 \|x_i\|_{V_N^\dagger}^2}\right)	\\
	&\leq& 2\dimvec \exp\left( - \frac{\eps^2}{2\sigma^2 \max_{k \in S} \|x_k\|_{V_N^\dagger}^2}\right)
	\end{eqnarray*}
	then the G-optimal design and the rounding (Lemma~\ref{lem:lem-alloc-psinv}) ensure that 
	$$ \max_{k \in S} \|x_k\|_{V_N^\dagger}^2 \leq (1+6\kappa) \dimfeat_S / N.$$  
	Thus
	$$ \bP\left(\|(\Theta - \widehat\Theta)^\T x_i\|_\infty \geq \veps\right) \leq 2\dimvec \exp\left( - \frac{N\eps^2}{2(1+6\kappa)\sigma^2\dimfeat_S}\right).$$
\end{proof}
\section{IMPLEMENTATION DETAILS AND ADDITIONAL EXPERIMENTS}
\label{sec:add_experiments}
In this section, we detail our experimental setup and provide additional experimental results. 
\subsection{Complexity and setup}
\label{sec:impl}
\paragraph{Time and memory complexity} The main computational cost of \gege{} (excepting calls to \optDesign{}) is the computation of the empirical gaps. This requires computing $\Mh(i,j;r)$ for any tuple $(i, j)$ of active arms and to temporarily store them. Computing the gaps results in a total $\cO(K^2\dimvec)$ time complexity and $\cO(K^2)$ memory complexity. Note that for the memory allocation, we can maintain the same arrays for the whole execution of the algorithm; thus, only cheap memory allocations are made after initialization. The overall computational complexity is reasonable as \gege{} is an elimination algorithm the computational cost reduces after rounds and we have proven that no more than $\lceil\log_2(1/\Delta_1)\rceil$ rounds are required in the fixed-confidence regime and only $\lceil\log_2(\dimfeat)\rceil$ rounds in the fixed-budget setting. For this reason, the computational complexity of a call to \optDesign{} has a limited impact in practice. We report below the average runtime on a personal computer with an ARM CPU 8GB RAM and 256GB SSD storage. The values are averaged over 50 runs. 
 
\begin{table}[htb]
\caption{Runtime of \gege{} recorded different instances.}
\centering
\begin{tabular}{|c| c |c|} 
 \hline
 $[K,\dimfeat, \dimvec]$  & \gege{}[$\delta=0.1$] & \gege{}[$T=500$] \\ [0.5ex] 
 \hline\hline
 $[10, 2, 2]$ & 6ms & 217ms  \\ \hline{}
 $[50, 8, 2]$ & 7ms & 464ms \\\hline{}
 $[100, 8, 4]$ & 545ms & 791ms  \\ \hline{}
 $[200, 8, 8]$ & 768ms & 1139ms\\ \hline{}
 $[500, 8, 8]$ & 1013ms & 2425ms \\ [1ex] 
 \hline
\end{tabular}
\label{table:runtime}
\end{table}
\paragraph{Setup} We have implemented the algorithms mainly in \texttt{python3} and \texttt{C++}. For each experiment, the value reported (sample complexity or probability of error) is averaged over 500 runs. For the experiments on synthetic instances we generate an instance satisfying the conditions reported in the main by first picking the $\dimfeat$ vectors (and thus $\Theta$) then the remaining arms are generated by sampling and normalizing some features from $\mathcal{U}([0, 1]^\dimfeat)$ to satisfy the contraints. For the real-world datasets, we normalize the features and (when mentioned) we use a least square to estimate a regression parameter $\widehat\Theta$ or we use the dataset as such (mis-specified setting i.e., without linearization). \pal{} is run with the same confidence bonus used in \cite{zuluaga_e-pal_2016} (which are tuned empirically) and for \algauername{}, we follow \cite{kone2023adaptive} and we use their confidence bonuses on a pair of arms, which was already suggested by \cite{auer_pareto_2016}.

\subsection{Additional experiments}
\label{sec:add_exp}
We provide additional experiments on synthetic and real-world datasets. \gege{} is evaluated both in the fixed-confidence and fixed-budget regimes.  
\paragraph{Multi-objective optimization of energy efficiency}
We use the energy efficiency dataset of \cite{misc_energy_efficiency_242}. This dataset is made for buildings' energy performance optimization. The efficiency of each building is characterized by $\dimvec=2$ quantities: the cooling load and the heating load. The heating load is the amount of energy that should be brought to maintain a building at an acceptable temperature, and the cooling load is the amount of energy that should be extracted from a building to sustain a temperature in an acceptable range. Ideally, both heating and cooling loads should be low for energy efficiency, and they are characterized by different factors like glazing area and the orientation of the building, amongst other parameters. \cite{misc_energy_efficiency_242} 
reported the simulated heating and cooling loads of $K=768$ buildings together with $\dimfeat = 8$ features characterizing 
each building, including surface, roof and wall areas, the relative compactness, overall height, etc.  The dataset was primarily made for multivariate regression, but we use it for linear PSI as the goal is to optimize simultaneously heating and cooling loads, which in general (and in this case), results in a Pareto front of 3 arms. 
 
We evaluate Algorithm~\ref{alg:gegefb} with a budget 
$T=10000$ and in the fixed-confidence we set $\delta=0.1$ for Algorithm~\ref{alg:gegeFc}. We report the results average over 500 runs on \myfig{fig:fig-ds1-supp-fb} and \myfig{fig:fig-ds1-supp-fc}.  In the fixed-confidence experiment, "Racing'' is the algorithm of \cite{auer_pareto_2016} for unstructured PSI. 
 \begin{minipage}{0.92\textwidth}
      \centering
  \begin{minipage}{0.42\linewidth}
          \begin{figure}[H]
              \includegraphics[width=\linewidth]{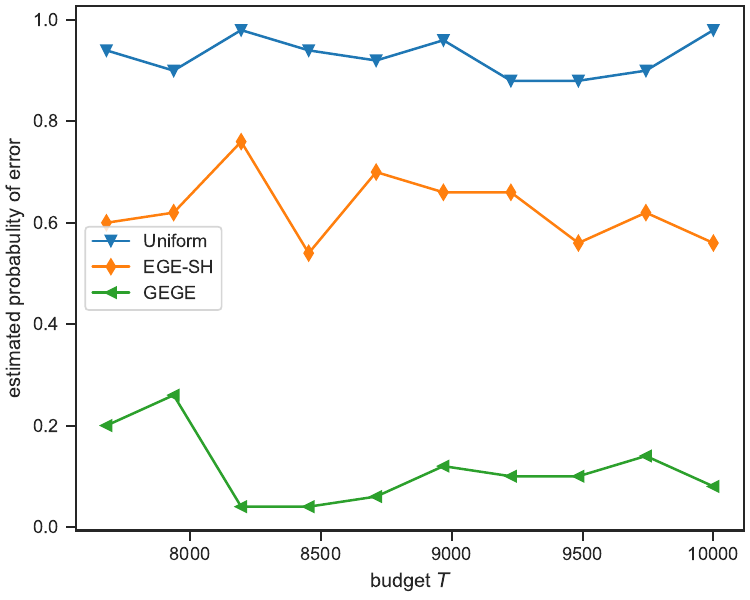}	
              \caption{Average probability of error on the energy efficiency dataset.}
              \label{fig:fig-ds1-supp-fb}
          \end{figure}
      \end{minipage}
\hspace{0.5cm}
      \begin{minipage}{0.42\linewidth}
          \begin{figure}[H]
      \includegraphics[width=\linewidth]{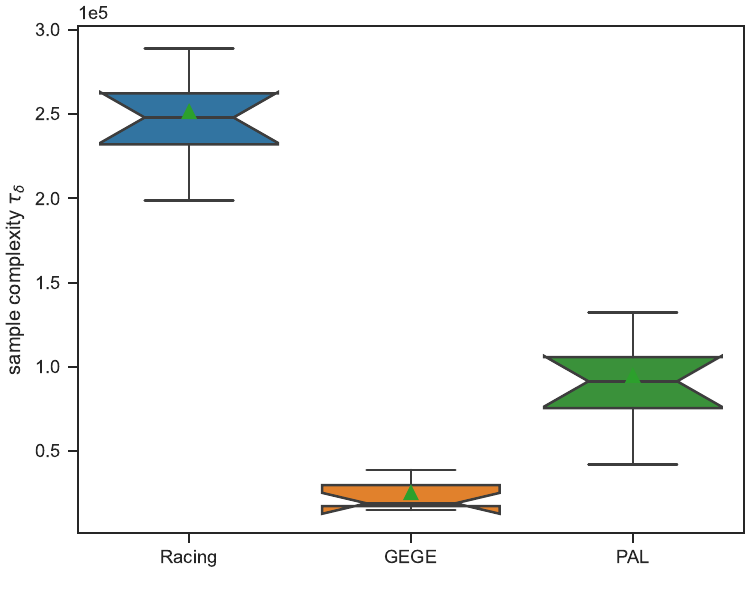}
              \caption{Sample complexity distribution on the energy efficiency dataset. }
              \label{fig:fig-ds1-supp-fc}
          \end{figure}
      \end{minipage}
      \end{minipage}

We observe that in both fixed-confidence and fixed-budget, \gege{} largely outperforms its competitors. It is worth noting that in the fixed-budget setting, as $K=768$, Uniform Allocation requires $T\geq 768$ to be run correctly, while EGE-SH requires $T\geq 7360$ to be able to initialize each arm, as they both ignore the linear structure. On the contrary \gegeFb{} just requires $T\geq \lceil \log(\dimfeat) \rceil$ which is negligible w.r.t $K=768$. Moreover, we observed that its probability of error is reasonable even for a budget $T<K$. 	
 \end{appendix}

\end{document}